\documentclass{article}

\usepackage{microtype}
\usepackage{graphicx}
\usepackage{booktabs} 
\usepackage{enumitem}
\usepackage{hyperref}
\usepackage{adjustbox}
\usepackage{hhline}
\usepackage{multirow}
\usepackage{subfig}

\usepackage[accepted]{icml2024}
\usepackage{amsmath}
\usepackage{amssymb}
\usepackage{mathtools}
\usepackage{amsthm}

\usepackage[capitalize,noabbrev]{cleveref}

\theoremstyle{plain}
\newtheorem{theorem}{Theorem}
\newtheorem{proposition}[theorem]{Proposition}
\newtheorem{identity}[theorem]{Identity}

\theoremstyle{definition}

\theoremstyle{remark}

\newcommand{\E}{\mathbb{E}}
\newcommand{\bas}[1]{\begin{align*}#1\end{align*}}

\newcommand{\ba}[1]{\begin{align}#1\end{align}}

\newcommand{\cdotv}{\boldsymbol{\cdot}}

\usepackage{stackengine}

\makeatletter
\newcommand{\distas}[1]{\mathbin{\overset{#1}{\kern\z@\sim}}}%

\newcommand{\cL}{\mathcal{L}}

\newcommand{\cN}{\mathcal{N}}

\newcommand{\beqs}{\vspace{0mm}\begin{eqnarray}}
\newcommand{\eeqs}{\vspace{0mm}\end{eqnarray}}
\newcommand{\barr}{\begin{array}}
\newcommand{\earr}{\end{array}}

\def\gL{{\mathcal{L}}}

\newcommand{\xv}{\boldsymbol{x}}

\newcommand{\zv}{\boldsymbol{z}}

\newcommand{\epsilonv}{\boldsymbol{\epsilon}}

\newcommand{\given}{\,|\,}

\usepackage[textsize=tiny]{todonotes}

\icmltitlerunning{Score identity Distillation}

\begin{document}

\twocolumn[
\icmltitle{
Score identity Distillation: Exponentially Fast Distillation of\\ Pretrained Diffusion Models for One-Step Generation
}




\begin{icmlauthorlist}
\icmlauthor{Mingyuan Zhou}{aaa,comp}
\icmlauthor{Huangjie Zheng}{aaa}
\icmlauthor{Zhendong Wang}{aaa}
\icmlauthor{Mingzhang Yin}{bbb}
\icmlauthor{Hai Huang}{comp}
\end{icmlauthorlist}

\icmlaffiliation{aaa}{The University of Texas at Austin}
\icmlaffiliation{bbb}{The University of Florida}
\icmlaffiliation{comp}{Google}

\icmlcorrespondingauthor{Mingyuan Zhou}{mingyuan.zhou@mccombs.utexas.edu}

\icmlkeywords{Diffusion Distillation, Score Matching, Deep Generative Models}

\vskip 0.3in
]



\printAffiliationsAndNotice{} 

\begin{abstract}

We introduce Score identity Distillation (SiD), an innovative data-free method that distills the generative capabilities of pretrained diffusion models into a single-step generator. SiD not only  facilitates an exponentially fast reduction in Fréchet inception distance (FID) during distillation but also approaches or even exceeds the FID performance of the original teacher diffusion models. By reformulating forward diffusion processes as semi-implicit distributions, we leverage three score-related identities to create an innovative loss mechanism. This mechanism achieves rapid FID reduction by training the generator using its own synthesized images, eliminating the need for real data or reverse-diffusion-based generation, all accomplished within significantly shortened generation time. Upon evaluation across four benchmark datasets, the SiD algorithm demonstrates high iteration efficiency during distillation and surpasses competing distillation approaches, whether they are one-step or few-step, data-free, or dependent on training data, in terms of generation quality. This achievement not only redefines the benchmarks for efficiency and effectiveness in diffusion distillation but also in the broader field of diffusion-based generation. The PyTorch implementation is available at \href{https://github.com/mingyuanzhou/SiD}{https://github.com/mingyuanzhou/SiD}.

\end{abstract}

\begin{figure*}[!t]
\begin{center}
\includegraphics[width=0.15\linewidth]{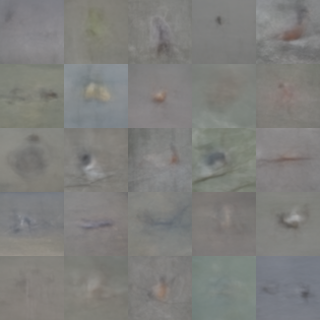}
\includegraphics[width=0.15\linewidth]{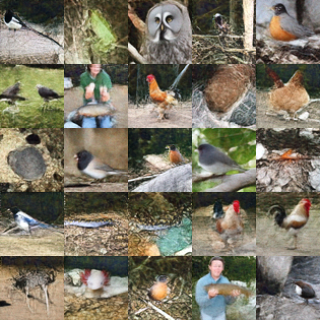}
\includegraphics[width=0.15\linewidth]{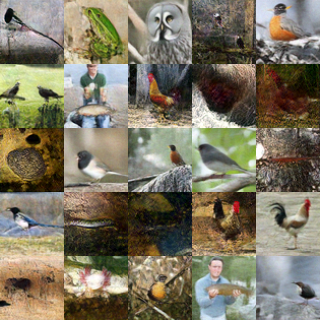}
\includegraphics[width=0.15\linewidth]{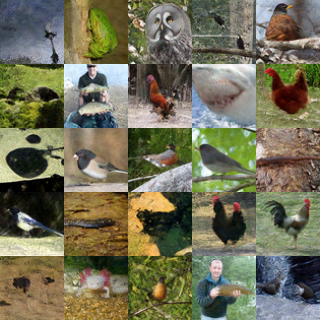}
\includegraphics[width=0.15\linewidth]{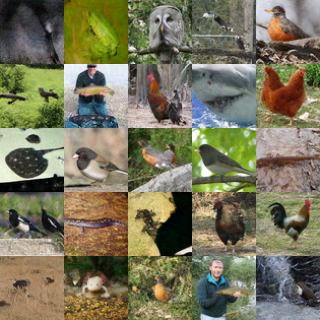}\\
\includegraphics[width=0.15\linewidth]{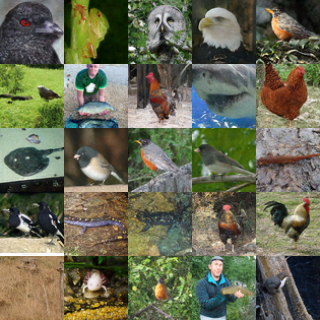}
\includegraphics[width=0.15\linewidth]{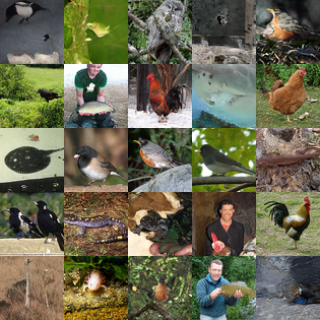}
\includegraphics[width=0.15\linewidth]{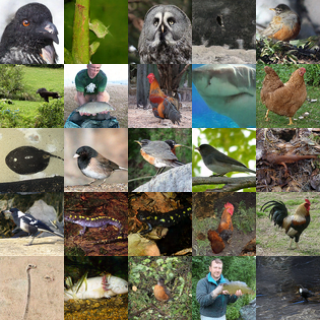}
\includegraphics[width=0.15\linewidth]{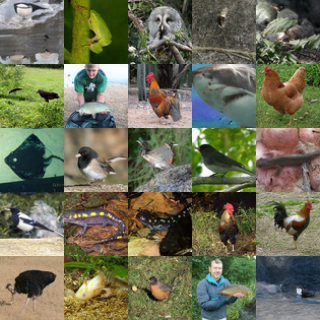}
\includegraphics[width=0.15\linewidth]{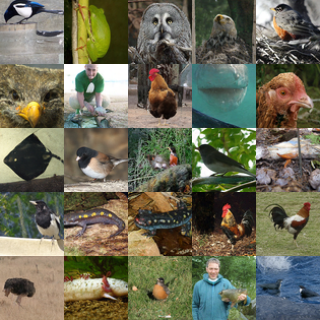}
\end{center}
 \vspace{-4.5mm}
 \caption{\small
Rapid advancements in the distillation of a pretrained ImageNet 64x64 diffusion model are shown using the proposed SiD method, with settings $\alpha=1.0$, a batch size of 1024, and a learning rate of 5e-6. The series of images, generated from the same set of random noises post-training the SiD generator with varying counts of synthesized images, illustrates progressions at 0, 0.1, 0.2, 0.5, 1, 2, 5, 10, 20, and 50 million images. These are equivalent to roughly 0, 100, 200, 500, 1K, 2K, 5K, 10K, 20K, and 49K training iterations respectively, organized from the top left to the bottom right. The associated FIDs for these iterations are 153.52, 34.83, 37.42, 18.08, 10.82, 7.74, 5.94, 4.49, 3.40, and 3.07, in order. The progression of FIDs is detailed in Fig.~\ref{fig:imagenet_1024} in the Appendix.
 }
 \label{fig:imagenet_progress}
 \vspace{-2.5mm}
\end{figure*}

\section{Introduction}

Diffusion models, also known as score-based generative models, have emerged as the leading approach for generating high-dimensional data \citep{sohl2015deep,song2019generative,ddpm}. These models are appreciated for their training simplicity and stability, their robustness against mode collapse during generation, and their ability to produce high-resolution, diverse, and photorealistic images \citep{dhariwal2021diffusion,ho2022cascaded,ramesh2022hierarchical,rombach2022high,saharia2022photorealistic,peebles2023scalable,zheng2023learning}. 

However, the process of generating data with diffusion models involves iterative refinement-based reverse diffusion, necessitating multiple iterations through the same generative network. This multi-step generation process, initially requiring hundreds or even thousands of steps, stands in contrast to the single-step generation capabilities of previous deep generative models such as variational auto encoders (VAEs) \citep{kingma2013auto,rezende2014stochastic} and generative adversarial networks (GANs) \citep{goodfellow2014generative,Karras2019stylegan2}, which only require forwarding the noise through the generation network once. Diffusion models necessitate multi-step generation, making them much more %
expensive at inference time. A wide variety of methods have been introduced to reduce the number of sampling steps during reverse diffusion, but they often still require quite a few number of function evaluations (NFE), such as 35 NFE for CIFAR-10 32x32 and 511 NFE for ImageNet 64x64 in EDM \citep{karras2022elucidating}, to achieve good~performance.

In this study, we aim to introduce a single-step generator designed to distill the knowledge on training data embedded in the score-estimation network of a pretrained diffusion model. To achieve this goal, we propose training the generator by minimizing a model-based score-matching loss between the scores of the diffused real data and the diffused generator-synthesized fake data distributions at various noise levels.
However, estimating this model-based score-matching loss, which is a form of Fisher divergence, at any given noise level proves to be intractable. To overcome this challenge, we offer a fresh perspective by viewing the forward diffusion processes of diffusion models through the lens of semi-implicit distributions. We introduce three corresponding score-related identities and illustrate their integration to formulate an innovative loss mechanism. This mechanism involves both score estimation and Monte Carlo estimation techniques to handle intractable expectations. Our method, designated as Score identity Distillation (SiD), is named to underscore its roots in these three identities.

We validate the effectiveness and efficiency of SiD across all four benchmark datasets considered in \citet{karras2022elucidating}: CIFAR-10 32x32, ImageNet 64x64, FFHQ 64x64, and AFHQv2 64x64. The SiD single-step generator is trained using the VP-EDM checkpoints as the teacher diffusion models. It achieves state-of-the-art performance across all four datasets in providing high-quality generation, measured by Fr\'echet inception distance (FID) \citep{heusel2017gans}, and also facilitates an exponentially fast reduction in FID as the distillation progresses. This is visually corroborated by Figs.~\ref{fig:imagenet_progress} and \ref{fig:cifar_alpha_image} and detailed in the experiments section.

\section{Related Work}

Significant efforts have been directed towards executing the reverse diffusion process in fewer steps. A prominent line of research involves interpreting the diffusion model through the lens of stochastic differential equations (SDE) or ordinary differential equations (ODE), followed by employing advanced numerical solvers for SDE/ODE \citep{ddim,liu2022pseudo,lu2022dpmsolver,zhang2023fast,karras2022elucidating,xue2023sa}. Despite these advancements, there remains a pronounced trade-off between reducing steps and preserving visual quality.
Another line of work considers the diffusion model within the framework of flow matching, applying strategies to simplify the reverse diffusion process into more linear trajectories, thereby facilitating larger step advancements \citep{liu2022flow,lipman2022flow}. %
To achieve generation within fewer steps, researchers also propose to truncate the diffusion chain and starting the generation from an implicit distribution instead of white Gaussian noise \citep{pandey2022diffusevae,zheng2023truncated,lyu2022accelerating} and combining it with GANs for faster generation \citep{xiao2021tackling, wang2023diffusiongan}. 

A unique avenue of research focuses on distilling the reverse diffusion chains \citep{luhman2021knowledge,salimans2022progressive,zheng2023fast,luo2023lcmlora}. \citet{salimans2022progressive} pioneered the concept of progressive distillation, %
which \citet{meng2023distillation} took further into the realm of guided diffusion models equipped with classifier-free guidance.
Subsequent advancements introduced consistency models \citep{song2023consistency} as an innovative strategy for distilling diffusion models, which 
promotes output consistency throughout the ODE trajectory. \citet{song2023improved} further enhanced the generation quality of these consistency models through extensive engineering efforts and new theoretical insights. Pushing the boundaries further, \citet{kim2023consistency} improved prediction consistency at any intermediate stage and incorporated GAN-based loss to elevate image quality. %
Extending
these principles, \citet{Luo2023LatentCM} applied consistency distillation techniques to text-guided latent diffusion models \citep{ramesh2022hierarchical}, facilitating efficient and high-fidelity text-to-image generation.

Recent research has focused on distilling diffusion models into generators capable of one or two step operations through adversarial training~\citep{Sauer2023AdversarialDD}. Following Diffusion-GAN~\citep{wang2023diffusiongan}, which trains a one-step generator by minimizing the Jensen--Shannon divergence (JSD) at each diffusion time step, \citet{Xu2023UFOGenYF} introduced UFOGen, which distills diffusion models using a time-step dependent discriminator, mirroring the initialization of the generator. UFOGen has shown proficiency in one-step text-guided image generation.
Text-to-3D synthesis, using a pretrained 2D text-to-image diffusion model, effectively acts as a distillation process, and leverages the direction indicated by the score function of the 2D diffusion model to guide the generation of various views of 3D objects~\citep{Poole2022DreamFusionTU,wang2023prolificdreamer}.
Building on this concept, Diff-Instruct \citep{luo2023diffinstruct} applies this principle to distill general pretrained diffusion models into a single-step generator, while SwiftBrush \citep{nguyen2023swiftbrush}  further illustrates its effectiveness in distilling pretrained stable diffusion models. Distribution Matching Distillation (DMD) of \citet{yin2023onestep} aligns closely with this principle, and further introduces an additional regression loss term to improve the quality of distillation. 

It is important to note that both Diff-Instruct and DMD are fundamentally aligned with the approach first introduced by Diffusion-GAN \citep{wang2023diffusiongan}. This method entails training the generator by aligning the diffused real and fake distributions. The primary distinction lies in whether the JSD or KL divergence is employed for any given noise level. From this perspective, the proposed SiD method adheres to the framework established by Diffusion-GAN and subsequently embraced by Diff-Instruct and DMD. 
However, SiD distinguishes itself by implementing a model-based score-matching loss, notably a variant of Fisher divergence, moving away from the traditional use of JSD or KL divergence applied to diffused real and fake distributions. Furthermore, it uncovers an effective strategy to approximate this loss, which is analytically intractable. In the sections that follow, we delve into both the distinctive loss mechanism and the method for its approximation, illuminating SiD's innovative~strategy.

\section{Forward Diffusion as Semi-Implicit Distribution: Exploring Score Identities}
The marginal of a mixture distribution can be expressed as 
$
p(\xv) = \int p(\xv\given \zv)p(\zv)\mathrm d\zv.
$
In cases where \( p(\xv\given \zv) \) is analytically defined and \( p(\zv) \) is straightforward to sample from, yet the marginal distribution is intractable or difficult to compute, we follow \citet{yin2018semi} to refer to it as a semi-implicit distribution. 
This semi-implicit framework and its derivatives have been widely used to develop flexible variational distributions with tractable parameter inference, as evidenced by a series of studies \citep{yin2018semi, molchanov2019doubly, hasanzadeh2019semi, titsias2019unbiased, sobolev2019importance, lawson2019energy, moens2021efficient, yu2023semiimplicit,yu2023hierarchical}.

In our study, we explore the vital role of semi-implicit distributions in the forward diffusion process. 
We observe that both the observed real data and the generated fake data adhere to semi-implicit distributions in this process. The gradients of their log-likelihoods, commonly known as scores, can be formulated as the expectation of certain random variables. These expectations are amenable to approximation through deep neural networks or Monte Carlo estimation. This reformulation of the scores is enabled through the application of the semi-implicit framework, allowing for the introduction of three critical identities pertinent to score estimation, as detailed subsequently.

\subsection{Forward Diffusions and Semi-Implicit Distributions}
The forward marginal of a diffusion model is an exemplary illustration of a semi-implicit distribution, expressed as: %
\begin{align}
\textstyle p_{\text{data}}(\xv_t) = \int q(\xv_t \given \xv_0)p_{\text{data}}(\xv_0) \, \mathrm d\xv_0, \label{eq:forward_diffusion}
\end{align}
where the forward conditional \( q(\xv_t \given \xv_0) \) is analytically defined, but the data distribution \( p_{\text{data}}(\xv_0) \) remains unknown and is typically represented through empirical samples.
In this paper, we focus on Gaussian diffusion models, where the forward conditional follows a Gaussian distribution:
\begin{align}
q(\xv_t \given \xv_0) = \mathcal{N}(a_t \xv_0, \sigma_t^2\mathbf{I}),\notag
\end{align}
with \( a_t \in [0,1] \).
To generate a diffused sample from \( \xv_0\sim p_{\text{data}}(\xv_0)\), reparameterization is often employed:
\begin{align}
\xv_t := a_t \xv_0 + \sigma_t\epsilonv_t, \quad \epsilonv_t \sim \mathcal{N}(\mathbf{0},\mathbf{I}).\notag
\end{align}
As \( a_t \) can be assimilated into the preconditioning of neural network inputs without sacrificing generality, we set \( a_t = 1 \) for simplicity, in line with \citet{karras2022elucidating}. 

While the exact form of \( p_{\text{data}}(\xv_t) \) and %
hence the score \( S(\xv_t) := \nabla_{\xv_t} \ln p_{\text{data}}(\xv_t) \) are not known, the score of the forward conditional \( q(\xv_t \given \xv_0) \) %
has an analytic expression as
\begin{align}
\nabla_{\xv_t} \ln q(\xv_t \given \xv_0) =\sigma_t^{-2}
{(\xv_0 - \xv_t)}= -\sigma_t^{-1}{\epsilonv_t}. %
\label{eq:Gaussian_conditional_score}
\end{align}

In this work, we explore an implicit generator \( p_{\theta}(\xv_g) \), parameterized by \(\theta\), which generates random samples as
$
\xv_g = G_{\theta}(\zv), \zv \sim p(\zv),
$
where \( G_{\theta}(\cdot) \) represents a neural network, parameterized by \(\theta\), that deterministically transforms noise \( \zv \sim p(\zv) \) into generated data \( \xv_g \). If the distribution of \( p_{\theta}(\xv_g) \) matches that of \( p_{\text{data}}(\xv_0) \), it then follows that the semi-implicit distribution
\begin{align}\textstyle
p_{\theta}(\xv_t) = \int q(\xv_t \given \xv_g)p_{\theta}(\xv_g) \, \mathrm d\xv_g \label{eq:semi-implicit}
\end{align}
would be identical to \( p_{\text{data}}(\xv_t) \) for any \( t \). Conversely, as proved in \citet{wang2023diffusiongan}, if \( p_{\theta}(\xv_t) \) coincides with \( p_{\text{data}}(\xv_t) \) for any \( t \), this implies a match between the generator distribution \( p_{\theta}(\xv_g) \) and the data distribution \( p_{\text{data}}(\xv_0) \).

\subsection{Score Identities}
In this paper, we illustrate that the semi-implicit distribution defined in \eqref{eq:semi-implicit} is characterized by three crucial identities, each playing a vital role in score-based distillation. The first identity concerns the diffused real data distribution, a well-established concept fundamental to denoising score matching. The second identity is analogous to the first but applies to diffused generator distributions.
The third identity, though not as widely recognized, is essential for the development of our proposed method. 

\begin{identity}[Tweedie's Formula for Diffused Real Data]\label{identity1}
Consider the semi-implicit distribution \( p_{\text{data}}(\xv_t) \) in \eqref{eq:forward_diffusion}, defined by diffusing real data. The expected value of \( \xv_0 \) given \( \xv_t \), in line with \( q(\xv_0 \given \xv_t) =\frac{q(\xv_t \given \xv_0)p_{\emph{\text{data}}}(\xv_0)}{p_{\emph{\text{data}}}(\xv_t)} \) as per Bayes' rule, is connected to the score of \( p_{\emph{\text{data}}}(\xv_t) \) as
\begin{align}
\resizebox{.905\columnwidth}{!}{$\textstyle
\E[\xv_0 \given \xv_t] = \int \xv_0 q(\xv_0 \given \xv_t) \, \mathrm d\xv_0 
= \xv_t + \sigma_t^2 \nabla_{\xv_t} \ln p_{\emph{\text{data}}}(\xv_t).$}
\label{eq:Tweedie_data}
\end{align}
\end{identity}

This identity, known as Tweedie's Formula \citep{robbins1992empirical, efron2011tweedie}, provides an estimate for the original data \( \xv_0 \) given \( \xv_t \), where \( \xv_t \) is generated as \( \xv_t \sim \mathcal{N}(\xv_0, \sigma_t^2 \mathbf{I}), \xv_0 \sim p_{\text{data}}(\xv_0) \). The significance of this relationship has been discussed in \citet{luo2022understanding} and \citet{chung2022improving}. An equivalent identity applies to the diffused fake data distribution.

\begin{identity}[Tweedie's Formula for Diffused Fake Data]\label{identity3}
For the semi-implicit distribution \( p_\theta(\xv_t) \) defined in \eqref{eq:semi-implicit}, resulting from diffusing fake data, the expected value of \( \xv_g \) given \( \xv_t \), following \( q(\xv_g \given \xv_t) =\frac{q(\xv_t \given \xv_g)p_\theta(\xv_g)}{p_\theta(\xv_t)} \) according to Bayes' rule, is associated with the score of \( p_\theta(\xv_t) \) as
\begin{align}
\resizebox{.905\columnwidth}{!}{$\textstyle
\E[\xv_g \given \xv_t] = \int \xv_g q(\xv_g \given \xv_t) \, \mathrm d\xv_g 
= \xv_t + \sigma_t^2 \nabla_{\xv_t} \ln p_\theta(\xv_t).$}
\label{eq:Tweedie_model}
\end{align}

\end{identity}

Transitioning from the initial two identities, we introduce a third identity that is crucial for the proposed computational methodology of score distillation, taking advantage of the properties of semi-implicit distributions.

\begin{identity}[Score Projection Identity]\label{identity:projected_score}
Given the intractability of \( \nabla_{\xv_t}\ln p_\theta(\xv_t) \), we introduce a projection vector to estimate the expected value of its product with the score:
\begin{align*}
&\E_{\xv_t \sim p_{\theta}(\xv_t)}\left[u^T(\xv_t) \nabla_{\xv_t}\ln p_{\theta}(\xv_t)\right]\notag\\ 
&\!\!= \E_{(\xv_t,\xv_g) \sim q(\xv_t \given \xv_g) p_{\theta}(\xv_g)}\left[u^T(\xv_t) \nabla_{\xv_t} \ln q(\xv_t \given \xv_g)\right]. %
\end{align*}\label{projected_score}
\vspace{-8mm}
\end{identity}

This identity was leveraged by \citet{vincent2011connection} to draw a parallel between the explicit score matching (ESM) loss, 
\begin{align}
\resizebox{.905\columnwidth}{!}{$\gL_{\text{ESM}} = \E_{\xv_t\sim p_\text{data}(\xv_t)}\left\|S_{\phi}(\xv_t)-\nabla_{\xv_t} \ln p_\text{data}(\xv_t)\right\|_2^2,$}\label{eq:esm_loss}
\end{align}
and denoising score matching (DSM) loss, given by
\begin{align}
\resizebox{.905\columnwidth}{!}{$\gL_{\text{DSM}} = \E_{q(\xv_t\given \xv_0)p_\text{data}(\xv_0)}\left\|S_{\phi}(\xv_t)-\nabla_{\xv_t} \ln q(\xv_t\given \xv_0)\right\|_2^2.$} \!\!\label{eq:dsm_loss}
\end{align}
Integrating the DSM loss with Unet architectures \citep{ronneberger2015unet} and stochastic-gradient Langevin dynamics \citep{welling2011bayesian} based reverse sampling, \citet{song2019generative} have elevated score-based, or diffusion, models to a prominent position in deep generative modeling.
Additionally, \citet{yu2023semiimplicit} used this identity for semi-implicit variational inference \citep{yin2018semi}, while \citet{yu2023hierarchical} applied it to refine multi-step reverse diffusion.

Distinct from these prior applications of this identity, we integrate it with two previously discussed identities. This fusion, combined with a model-based score-matching loss, culminates in a unique loss mechanism facilitating single-step distillation of a pretrained diffusion model.

\section{SiD: Score identity Distillation}

In this section, we introduce the model-based score-matching loss %
as the theoretical basis for our distillation loss. We then demonstrate how the three identities previously discussed can be fused to approximate this loss.

\subsection{Model-based Explicit Score Matching (MESM) }

Assuming the existence of a diffusion model for the data, with parameter \( \phi \) pretrained to estimate the score \( \nabla_{\xv_t}\ln p_\text{data}(\xv_t) \), we use the following approximation:
\begin{align*}\textstyle 
\nabla_{\xv_t} \ln p_\text{data}(\xv_t) \approx S_{\phi}(\xv_t):= {\sigma_t^{-2}}({f_{\phi}(\xv_t,t)-\xv_t}).
\end{align*}
In other words, we adopt \( f_{\phi}(\xv_t,t)\approx \E[\xv_0\given \xv_t] \) as our approximation, according to \eqref{eq:Tweedie_data} in Identity 1. Our goal is to distill the knowledge encapsulated in \( \phi \), extracting which for data generation typically requires many iterations through the same network $f_{\phi}(\cdotv,\cdotv)$.

We %
use the pretrained score $S_{\phi}(\xv_t)$ to construct our distillation loss. Our aim is to train \( G_{\theta} \) to distill the iterative, multi-step reverse diffusion process into a single network evaluation step. For a specific reverse diffusion time step \( t \sim p(t)\), we define the theoretical distillation loss as
\begin{align}\textstyle 
&\gL_{\theta} = \E_{\xv_t\sim p_{\theta}(\xv_t)}[
\|\delta_{\phi,\theta}(\xv_t)\|_2^2],\label{eq:MESM}\\
&\delta_{\phi,\theta}(\xv_t):= S_{\phi}(\xv_t) - \nabla_{\xv_t}\ln p_{\theta}(\xv_t) .\label{eq:scorediff}
\end{align}
We refer to $\delta_{\phi,\theta}(\xv_t)$ as score difference and
designate its expected L2 norm $\gL_{\theta}$ as the model-based explicit score-matching (MESM) loss, also known in the literature as a Fisher divergence \citep{lyu2009interpretation,holmes2017assigning,yang2019variational,yu2023semiimplicit}. This differs from the ESM loss defined in \eqref{eq:esm_loss} as the expectation is computed with respect to the diffused fake data distribution \( p_{\theta}(\xv_t) \) rather than diffused real data distribution \( p_\text{data}(\xv_t) \).

A common assumption is that the performance of the student model used for distillation will be limited by the outcomes of reverse diffusion using \( S_{\phi}(\xv_t) \), the teacher model. However, our results demonstrate that the student model, utilizing single-step generation, can indeed exceed the performance of the teacher model, EDM of \citet{karras2022elucidating}, which relies on iterative refinement. This indicates that the aforementioned hypothesis might not necessarily hold true. It implies that reverse diffusion could accumulate errors throughout its process, even with very fine-grained reverse steps and the use of advanced numerical solvers designed to counteract error accumulations.~\looseness=-1

\subsection{Loss Approximation based on Identities 1 and 2} \label{sec:4.2}
To estimate the score $\nabla_{\xv_t}\ln p_{\theta}(\xv_t)$, an initial thought would be to adopt a deep neural network-based approximation $f_{\psi}(\xv_t,t)\approx \E[\xv_g\given \xv_t]$ by \eqref{eq:Tweedie_model}, which can be trained with the usual diffusion or denoising score-matching loss as
\ba{
\resizebox{.89\columnwidth}{!}{$
\min_{\psi} %
\E_{ %
q(\xv_t\given \xv_g,t)p_{\theta}(\xv_g)}[{\gamma(t)} %
\|f_{\psi}(\xv_t,t)-\xv_g\|_2^2 ],$}
\label{eq:obj-psi}
}
where the timestep distribution $t\sim p(t)$ %
and 
weighting function $\gamma(t)$ %
can be defined as in
\citet{karras2022elucidating}. Assuming $\xv_t\sim q(\xv_t\given \xv_g),\, \xv_g = G_{\theta}(\zv),\, \zv\sim p(\zv)$, the optimal solution $\psi^*(\theta)$, which depends on the generator distribution determined by $\theta$, satisfies $$f_{\psi^*(\theta)}(\xv_t,t)=\E[\xv_g\given \xv_t]=\xv_t + \sigma_t^2 \nabla_{\xv_t} \ln p_\theta(\xv_t)$$ and we can express the score difference defined in \eqref{eq:scorediff} as \ba{%
\delta_{\phi,\psi^*(\theta)}(\xv_t)=\sigma_t^{-2}(f_{\phi}(\xv_t,t)-
f_{\psi^*(\theta)}(\xv_t,t)).\label{eq:scoredifference}
}
As $\psi^*(\theta)$ depends on $\theta$, 
the minimization of $\mathcal L_{\theta}$ in \eqref{eq:MESM} could potentially be cast as a bilevel optimization problem \citep{ye1997exact,hong2023two,shen2023penalty}.

It is tempting to 
estimate the score difference $\delta_{\phi,\psi^*(\theta)}(\xv_t)$ 
using an approximated score difference defined as
\ba{
{\delta}_{\phi,\psi}(\xv_t) := \sigma_t^{-2}(f_{\phi}(\xv_t,t)-
f_{\psi}(\xv_t,t)),
\label{eq:delta0}
}
which means we approximate $\psi^*(\theta)$ with $\psi$, ignoring its dependence on $\theta$, and define an approximated MESM loss $\gL_{\theta}^{(1)}$
as 
\ba{
\gL_{\theta} \approx \gL_{\theta}^{(1)}:=\E_{\xv_t\sim p_{\theta}(\xv_t)} \left[\|{\delta}_{\phi,\psi}(\xv_t)\|_2^2 \right].\label{eq:L_theta1}
}

However, defining the score approximation error as
\ba{
\triangle_{\psi,\psi^*(\theta)}(\xv_t) := %
\sigma_t^{-2}({f_{\psi}(\xv_t,t)-f_{\psi^*(\theta)}(\xv_t,t)}),
\label{eq:delta1}
}
 we have
$
{\delta}_{\phi,\psi}(\xv_t) %
= {\delta}_{\phi,\psi^*(\theta)}(\xv_t) - \triangle_{\psi,\psi^*(\theta)}(\xv_t)
$
and %
\ba{
\gL_{\theta}^{(1)}
=\gL_{\theta} \,+\,&\E_{p_{\theta}(\xv_t)} \big[ \|\triangle_{\psi,\psi^*(\theta)}(\xv_t)\|_2^2 \notag\\
&-2\triangle_{\psi,\psi^*(\theta)}(\xv_t)^T\delta_{\phi,\psi^*(\theta)}(\xv_t) \big].
\label{eq:L1}
}
Therefore, how well $\gL_{\theta}^{(1)}$ approximates the true loss $\gL_{\theta}$ heavily depends on not only the score approximation error $\triangle_{\psi,\psi^*(\theta)}(\xv_t)$ but also the score difference $\delta_{\phi,\psi^*(\theta)}(\xv_t)$. For a given $\theta$, although one can control $\triangle_{\psi,\psi^*(\theta)}(\xv_t)$ by minimizing \eqref{eq:obj-psi}, it would be difficult to ensure that influence of the score difference $\delta_{\phi,\psi^*(\theta)}(\xv_t)$ %
would not dominate the true loss $\gL_{\theta}$, especially during the intial phase of training when $p_{\theta}(\xv_t)$ does not match $p_\text{data}(\xv_t)$ well.

This concern is confirmed through the distillation of EDM models pretrained on CIFAR-10, employing a loss estimated via reparameterization and Monte Carlo estimation as
\ba{
&\hat{\gL}_{\theta}^{(1)} = \|{\delta}_{\phi,\psi}(\xv_t)\|_2^2,~\\
&\xv_t=\xv_g+\sigma_t\epsilonv_t, \,\,\epsilonv_t %
\sim 
\mathcal{N}(\mathbf{0},\mathbf{I})\\&\xv_g=G_{\theta}(\sigma_{\text{init}}\zv),~ \zv %
\sim 
\mathcal{N}(\mathbf{0},\mathbf{I}).\label{eq:repara}
}
This loss fails to yield meaningful results.
Below, we present a toy example that highlights a failure case when using $\hat{\mathcal{L}}_{\theta}^{(1)}$ as the loss function to optimize $\theta$.

\begin{proposition}[An example failure case]
Suppose $p_\text{data}(x_0) = \mathcal{N}(0,1)$, $p_{data}(x_t)=\mathcal{N}(0,1+\sigma_t^2)$, $q(x_t\given x_g) = \mathcal N(x_g,\sigma_t^2)$, and $p_\theta(x_g) = \mathcal{N}(\theta,1)$. Assume $\psi^*(\theta)=\theta$ and $f_{\psi}(x_t,t) = x_t(1+\sigma_t^2)^{-1} + \psi \sigma_t^2(1+\sigma_t^2)^{-1}$.
Then we have
\begin{enumerate}[label=(\roman*)]
\item $\delta_{\phi,\psi^*(\theta)}(x_t) = - \frac{\theta}{1+\sigma_t^2}$, $\delta_{\phi,\psi}(x_t) = - \frac{\psi}{1+\sigma_t^2}$;
\item $\mathcal{L}_{\theta}=\frac{\theta^2}{(1+\sigma_t^2)^2}$, $\hat{\mathcal{L}}_{\theta}^{(1)} = \frac{\psi^2}{(1+\sigma_t^2)^2} $.
\end{enumerate}
\label{prop:toy}
\end{proposition}

The proof %
is presented in \Cref{sec:toy}. The example in this proposition shows that although minimizing the objective $\mathcal{L}_{\theta}$ leads to the optimal generator parameter $\theta^*=0$, the loss $\hat{\mathcal{L}}_{\theta}^{(1)}$ would provide no meaningful gradient towards $\theta^*$. ~\looseness=-1

\subsection{Loss Approximation via Projected Score Matching}
We provide an alternative formulation of the MESM loss: 
\begin{theorem}[Projected Score Matching]\label{thm:project_SM}
The MESM loss in~\eqref{eq:MESM} can be equivalently expressed as
 \ba{
 \resizebox{.98\columnwidth}{!}{$
\gL_{\theta} %
= \E_{q(\xv_t\given \xv_g,t)p_{\theta}(\xv_g)}\left[ %
\sigma_t^{-2}{\delta_{\phi,\psi^*(\theta)}(\xv_t)^T(f_{\phi}(\xv_t,t)-\xv_g)}\right]$}
.
\label{eq:project_SM}
}
\end{theorem}

The proof, based on Identity~\ref{identity:projected_score}, is deferred to \Cref{sec:proof}.
We approximate the loss by substituting $\psi^*(\theta)$ in \eqref{eq:project_SM} with its approximation $\psi$, %
leading to an approximated loss $\gL_{\theta}^{(2)}$~as
\ba{
\gL_{\theta}^{(2)} &
\resizebox{.905\columnwidth}{!}{$
=\E_{q(\xv_t\given \xv_g,t)p_{\theta}(\xv_g)}\left[ %
\sigma_t^{-2}{\delta_{\phi,\psi}(\xv_t)^T(f_{\phi}(\xv_t,t)-\xv_g)}\right]$}\notag\\
&\!\!= \gL_{\theta}-\,\E_{q(\xv_t\given \xv_g,t)p_{\theta}(\xv_g)}\notag\\
&~~~~~~~~~~\left[ %
\sigma_t^{-2}{\triangle_{\psi,\psi^*(\theta)}(\xv_t)^T(f_{\phi}(\xv_t,t)-\xv_g)}\right].\label{eq:L2}
}
Comparing \eqref{eq:L2} to \eqref{eq:L1} indicates that $\gL_{\theta}^{(2)}$ is directly influenced by neither the norm $\|\triangle_{\psi,\psi^*(\theta)}(\xv_t)\|_2^2$ nor the score difference $\delta_{\phi,\psi^*(\theta)}(\xv_t)$ given by \eqref{eq:scoredifference}. Initially in training, the discrepancy between the estimated and actual scores for the generator distribution may amplify the value of $\triangle_{\psi,\psi^*(\theta)}(\xv_t)$, whereas the difference between the pre-trained score for the real data distribution and the actual score for the generator distribution may inflate $\delta_{\phi,\psi^*(\theta)}(\xv_t)$. By contrast, the term $f_{\phi}(\xv_t,t)-\xv_g$ within \eqref{eq:L2} reflects the efficacy of the pre-trained model in denoising corrupted fake data, which tends to be more stable.

Let's verify the failure case for $\gL_{\theta}^{(1)}$ and see whether it is still the case for $\gL_{\theta}^{(2)}$.

\begin{proposition}
 Under the setting of \Cref{prop:toy}, the gradient of loss $\gL_{\theta}^{(2)}$ can be estimated as
\bas{
\nabla_\theta \hat{L}_{\theta}^{(2)} = - ({1+\sigma_t^2})^{-1}{\delta}_{\phi,\psi}(\xv_t) \nabla_{\theta} G_{\theta}(\sigma_{\text{init}}\zv),
}
 which involves the product of the approximated score difference ${\delta}_{\phi,\psi}(\xv_t)=- \frac{\psi}{1+\sigma_t^2}$ and the gradient of the generator.
\label{prop:toy2}
\end{proposition}
We note the product of the approximated score difference and $\nabla_{\theta} G_{\theta}(\sigma_{\text{init}}\zv)$ is used to construct the loss for Diff-Instruct \citep{luo2023diffinstruct}, which has been shown to be able to distill a pretrained diffusion model with satisfactory performance. Thus for the toy example where $\gL_{\theta}^{(1)}$ fails, using $\gL_{\theta}^{(2)}$ can provide useful gradient to guide the generator.

\subsection{Fused Loss of SiD}

Examining $\gL_{\theta}^{(2)}$ and $\gL_{\theta}^{(1)}$ unveils their interconnections: 
\ba{ 
\gL_{\theta}^{(2)} 
=~&\gL_{\theta}^{(1)}+\E_{\xv_g\sim p_{\theta}(\xv_g)}\E_{ \xv_t\sim q(\xv_t\given \xv_g,t)}\notag\\
&\left[ 
{\sigma_t^{-2}}{{\delta}_{\phi,\psi}(\xv_t)^T(f_{\psi}(\xv_t,t)-\xv_g)}
\right]. %
}
Empirically, while $\hat{\gL}_{\theta}^{(1)}$ fails, our distillation experiments on CIFAR-10 reveal that $\hat{\gL}_{\theta}^{(2)}$ performs well in terms of Inception Score (IS), but yields poor FID.
This outcome is illustrated in the visualizations for 
$\alpha=0$ in Figs. \ref{fig:cifar_alpha_image} and \ref{fig:cifar_ablate}.

 Visual inspection 
indicates that the generated images are darker in comparison to the training images. 
Given that $\hat{\gL}_{\theta}^{(1)}$ fails while $\hat{\gL}_{\theta}^{(2)}$ shows promis, albeit with poor FID due to mismatched color, we hypothesize that the difference term $$\hat{\gL}_{\theta}^{\triangle} = \hat{\gL}_{\theta}^{(2)}-\hat{\gL}_{\theta}^{(1)} = {\sigma_t^{-2}}{\delta}_{\phi,\psi}(\xv_t)^T(f_{\psi}(\xv_t,t)-\xv_g)$$ directs the gradient towards the desired direction. 

Thus we are propelled to consider the loss
\ba{
\gL_{\theta}^{(2)}-\alpha \gL_{\theta}^{(1)} = (1-\alpha)\gL_{\theta}^{(1)}+\gL_{\theta}^{\triangle}.\label{eq:loss_combine}
}

We empirically find that setting \(\alpha\in[-0.25, 1.2]\) produces visually coherent images, with \(\alpha\in[0.75, 1.2]\) typically leading to superior results, as shown in Figs.\,\ref{fig:cifar_alpha_image} and 
\ref{fig:cifar_ablate}.

In summary, the weighted loss is expressed as
\ba{
&\textstyle \tilde{L}_{\theta}(\xv_t,t,\phi,\psi)
=(1-\alpha) \frac{\omega(t)}{{\sigma_t^4} } \|f_{\phi}(\xv_t,t)-f_{\psi}(\xv_t,t)\|_2^2\notag\\
&\textstyle+\frac{\omega(t)}{{\sigma_t^4} } (f_{\phi}(\xv_t,t)-f_{\psi}(\xv_t,t))^T(f_{\psi}(\xv_t,t)-\xv_g),
\label{eq:obj-theta}
}
where $\xv_t$ is generated as in \eqref{eq:repara} and $\omega(t)$ are weighted coefficients that need to be specified.
To compute the gradient of the above equation, 
SiD backpropagates the gradient through both $\phi$ and $\psi$ by calculating 
two score gradients ($i.e.$, gradients of scores) as 
\begin{equation}
\begin{aligned}
&\textstyle\nabla_{\theta}f_{\phi}(\xv_t,t)=\frac{\partial f_{\phi}(\xv_t,t)}{\partial \xv_t}\nabla_{\theta} G_{\theta}(\sigma_{\text{init}}\zv)\\
&\textstyle\nabla_{\theta}f_{\psi}(\xv_t,t)=\frac{\partial f_{\psi}(\xv_t,t)}{\partial \xv_t}\nabla_{\theta} G_{\theta}(\sigma_{\text{init}}\zv).
\end{aligned}\label{eq:scoregrad}
\end{equation}
This feature distinguishes SiD from Diff-Instruct and DMD that do not use score gradients $\frac{\partial f_{\phi}(\mathbf{x}_t,t)}{\partial \xv_t}$ and $\frac{\partial f_{\psi}(\xv_t,t)}{\partial \xv_t}$.

\subsection{Noise Weighting and Scheduling}
The proposed SiD algorithm iteratively updates the score estimation parameters \(\psi\), given \(\theta\), following %
\eqref{eq:obj-psi}, and updates the generator parameters \(\theta\), given \(\psi\), as per %
\eqref{eq:obj-theta}. This alternating update scheme aligns with related approaches \citep{wang2023diffusiongan,luo2023diffinstruct,yin2023onestep}.
Consequently, we largely adopt the methodology outlined by \citet{luo2023diffinstruct} and \citet{yin2023onestep} for setting model parameters, including weighting coefficients \(\omega(t)\) and the distribution of \(t\sim p(t)\). Specifically, denoting %
\(C\) as the total pixel count of an image and \(\|\cdot\|_{1,sg}\) as the L1 norm combined with the stop gradient operation, %
we define
\ba{
\omega(t) = C\sigma_t^4 /\|\xv_g-f_{\phi}(\xv_t,t)\|_{1,sg}~~.\label{eq:weight}
}
Choosing $\sigma_{\min} = 0.002$, $\sigma_{\max}=80$, $\rho=7.0$, and $t_{\max}\in[0,1000]$, we sample $t\sim \mbox{Unif}[0,t_{\max}/1000]$ %
 and
 define the  noise levels as
\ba{
\sigma_t = \big(\sigma_{\max}^{\frac{1}{\rho}}+\textstyle(1-t)(\sigma_{\min}^{\frac{1}{\rho}}-\sigma_{\max}^{\frac{1}{\rho}})\big)^{\rho}.\label{eq:sigma_t}
}

The distillation process is outlined in \Cref{alg:sid}. The one-step generation procedure is straightforward:
$\xv= G_\theta(\sigma_{\text{init}}\zv),~ \zv \sim \cN(\mathbf{0},  \mathbf{I}),$
where \(\sigma_{\text{init}}\), by default set to 2.5, remains consistent throughout  distillation and generation.

\begin{figure}[!th]
\centering
\begin{minipage}{\linewidth} %

\captionof{table}{\small Comparison of various deep generative models trained on {CIFAR-10} without label conditioning. The best and second-best one/few-step generators under the FID or IS metric are highlighted with \textbf{bold} and \textit{\textbf{italic bold}}, respectively.\label{tab:cifar10_uncond}}
\begin{adjustbox}{width=1\linewidth,
center}
\begin{tabular}{clccc}
 \toprule[1.5pt]
Family & Model &NFE & FID~($\downarrow$) & IS~($\uparrow$)\\ %
 \midrule
 Teacher & VP-EDM~\citep{karras2022elucidating} & 35 &{1.97} & 9.68\\ %
 \midrule
 \multirow{6}{*}{Diffusion} & DDPM~\citep{ddpm} & 1000 & 3.17 & 9.46\scriptsize{$\pm$0.11} \\ %
 & DDIM~\citep{ddim} & 100 & 4.16 & \\
 & DPM-Solver-3~\citep{lu2022dpmsolver} & 48 &  2.65& \\
 &VDM~\citep{kingma2021variational} & 1000 & 4.00 &  \\
 &iDDPM~\citep{nichol2021improved} & 4000 & 2.90 &  \\
 &HSIVI-SM~\citep{yu2023hierarchical} & 15 &4.17 &\\
  &TDPM~\citep{zheng2023truncated} & 5 & 3.34 & \\
 &TDPM+~\citep{zheng2023truncated} & 100 & 2.83 & 9.34\\
 &VP-EDM+LEGO-PR~\citep{zheng2023learning} &35&1.88&9.84 \\
 \midrule
\multirow{25}{*}{One Step}
& NVAE~\citep{vahdat2020nvae} & 1 & 23.5 & \\
&StyleGAN2+ADA~\citep{Karras2019stylegan2} &1&5.33\scriptsize{$\pm$0.35}& \textbf{\textit{10.02}}\scriptsize{$\pm$0.07}\\
&StyleGAN2+ADA+Tune~\citep{Karras2019stylegan2} &1&2.92\scriptsize{$\pm$0.05}& 9.83\scriptsize{$\pm$0.04}\\
&CT-StyleGAN2~\citep{zheng2021exploiting} &1&2.9$\pm$\scriptsize{0.4}& \textbf{10.1}$\pm$\scriptsize{0.1}\\
&StyleGAN2+ DiffAug~\citep{zhao2020differentiable} & 1 & 5.79 &\\ %
& ProjectedGAN~\citep{sauer2021projected} & 1 & 3.10 & \\
 & DiffusionGAN~\citep{wang2023diffusiongan} & 1 & 3.19 &\\ %
 &Diffusion ProjectedGAN~\citep{wang2023diffusiongan} & 1 & 2.54 & \\
 & KD~\citep{luhman2021knowledge} &1& 9.36 & \\
 & TDPM ~\citep{zheng2023truncated} & 1 &7.34& \\
 & PD~\citep{salimans2022progressive} & 1 & 8.34 & 8.69\\
 &Score Mismatching~\citep{ye2023score} & 1 & 8.10 &\\
 & 2-ReFlow~\citep{liu2022flow} & 1 & 4.85 & 9.01\\
 & DFNO~\citep{zheng2023fast} & 1 & 3.78 & \\
 & CD-LPIPS \citep{song2023consistency} & 1 & 3.55 & 9.48 \\
 & iCT~\citep{song2023improved} & 1 & 2.83 & 9.54 \\
 & iCT-deep~\citep{song2023improved} & 1 & 2.51 & 9.76 \\ 
&G-distill~\citep{meng2023distillation} ($w$=0.3) & 1 & 7.34 & 8.9\\
&GET-Base~\citep{geng2024one}  & 1 & 6.91 & 9.16\\
&Diff-Instruct~\citep{luo2023diffinstruct} & 1 & 4.53 & 9.89\\
&StyleGAN2+ADA+Tune+DI~\citep{luo2023diffinstruct} &1&2.71 & 9.86\scriptsize{$\pm$0.04}\\
&PID~\citep{tee2024physics}  & 1 & 3.92 &9.13 \\
&TRACT~\citep{berthelot2023tract} & 1 & 3.78 & \\
& DMD~\citep{yin2023onestep} & 1 & 3.77 & \\
&CTM~\citep{kim2023consistency}&1&\textbf{\textit{1.98}}&\\
&SiD (ours), $\alpha=1.0$ & 1 & 2.028 $\pm$\scriptsize{0.020}  & \textbf{\textit{10.017}}$\pm$\scriptsize{0.047}\\ 
&SiD (ours), $\alpha=1.2$ & 1 & \textbf{1.923}$\pm$\scriptsize{0.017} & {9.980}$\pm$ \scriptsize{0.042}\\ 
 \bottomrule[1.5pt]
\end{tabular}
\end{adjustbox}%
\captionsetup[subfloat]{} 
 \centering
 {\includegraphics[width=.75\linewidth]{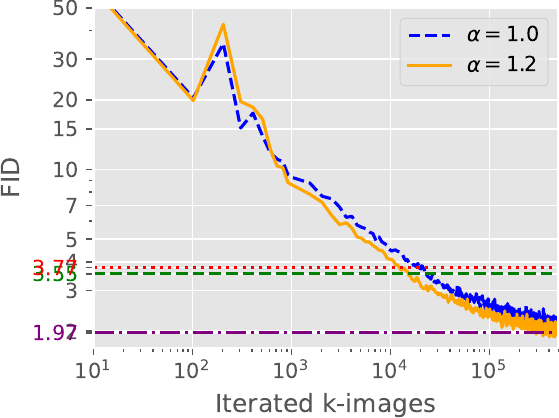}}
 \vspace{-3mm}
\captionof{figure}{\small 
Evolution of FIDs for the SiD generator during the distillation of the EDM teacher model pretrained on CIFAR-10 (unconditional), using $\alpha=1.0$ or $\alpha=1.2$ and a batch size of 256.
 The performance of EDM, along with  DMD and Diff-Instruct, is depicted with horizontal lines in purple, green, and red, respectively. %
 }
 \label{fig:convergence_speed}
\vspace{-4.5mm}
\end{minipage}
\end{figure}

\section{Experimental Results}
Initially, we demonstrate the capability of the Score identity Distillation (SiD)  generator to rapidly train and generate photo-realistic images by leveraging the pretrained score network and its own synthesized fake images. Subsequently, we conduct an ablation study to investigate the impact of the parameter $\alpha$ and discuss the settings of several other parameters. Through extensive experimentation, we assess both the effectiveness and efficiency of SiD in the context of diffusion-based image generation.

\begin{figure}[t]
\centering
\begin{minipage}{\linewidth} 
\captionof{table}{\small 
Analogous to Table \ref{tab:cifar10_uncond} for 
{CIFAR-10} (conditional).\label{tab:cifar10_cond}}
\begin{adjustbox}{width=1\linewidth,
center}
\begin{tabular}{clcc}
 \toprule[1.5pt]
Family & Model &NFE & FID~($\downarrow$)\\ %
 \midrule
 \multirow{1}{*}{Teacher} %
 & VP-EDM~\citep{karras2022elucidating} & 35& {1.79}\\ 
 \midrule
 \multirow{3}{0.2\linewidth}{Direct generation}  %
 & BigGAN~\citep{brock2018large} & 1 & 14.73\\ %
 &StyleGAN2+ADA~\citep{Karras2019stylegan2} &1&3.49\scriptsize{$\pm$0.17}\\
 &StyleGAN2+ADA+Tune~\citep{Karras2019stylegan2} &1&2.42\scriptsize{$\pm$0.04}\\
 \midrule
\multirow{8}{*}{Distillation}
 &GET-Base~\citep{geng2024one}  & 1 & 6.25 \\
&Diff-Instruct~\citep{luo2023diffinstruct}&1 & 4.19 \\
&StyleGAN2+ADA+Tune+DI~\citep{luo2023diffinstruct} &1&2.27\\
& DMD~\citep{yin2023onestep} &1 & 2.66 \\
& DMD (\textit{w.o.} KL)~\citep{yin2023onestep} &1 & 3.82 \\
& DMD (\textit{w.o.} \textit{reg}.)~\citep{yin2023onestep} &1 & 5.58 \\
&CTM~\citep{kim2023consistency}&1&\textbf{\textit{1.73}}\\
&SiD (ours), $\alpha=1.0$&1 & 1.932$\pm$\scriptsize{0.019} 
\\ 
&SiD (ours), $\alpha=1.2$ &1 & \textbf{{1.710}}$\pm$\scriptsize{0.011} \\ 
 \bottomrule[1.5pt]
\end{tabular}
\end{adjustbox}%
\captionsetup[subfloat]{} 
 \centering
 {\includegraphics[width=.7\linewidth]{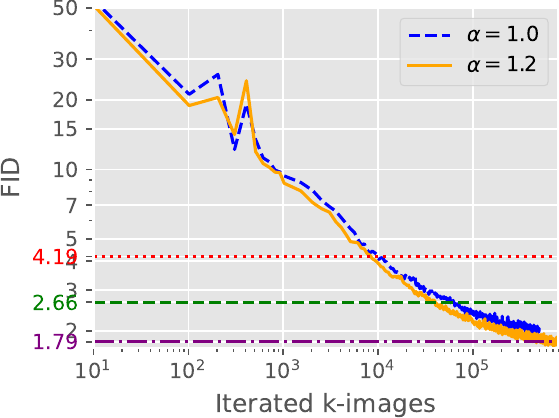}\label{fig:first_img}}
 \vspace{-3mm}
\captionof{figure}{\small Analogous to Fig.~\ref{fig:convergence_speed} for CIFAR-10 (conditional).
 }
 \label{fig:cifar_cond_convergence_speed}
\vspace{-3mm}
\end{minipage}
\end{figure}

\textbf{Datasets. } To thoroughly assess the effectiveness of SiD, we utilize four representative datasets considered in \citet{karras2022elucidating}, including CIFAR-10 $32 \times 32$ (cond/uncond) \citep{cifar10}, ImageNet $64 \times 64$ \citep{deng2009imagenet}, FFHQ $64 \times 64$ \citep{karras2019style}, and AFHQ-v2 $64 \times 64$ \citep{choi2020stargan}. 

\textbf{Evaluation protocol. } We measure image generation quality using FID and Inception Score (IS;~\citet{salimans2016improved}). Following \citet{karras2019style, karras2022elucidating}, we measure FIDs using 50k generated samples, with the training set used by the EDM teacher model\footnote{\url{https://github.com/NVlabs/edm}} as reference. We also consider Precision and Recall \citep{kynkaanniemi2019improved} when evaluating SiD on ImageNet 64x64, where we use a predefined reference batch\footnote{\url{https://openaipublic.blob.core.windows.net/diffusion/jul-2021/ref_batches/imagenet/64/VIRTUAL_imagenet64_labeled.npz}} to compute both metrics\footnote{\url{https://github.com/openai/guided-diffusion/tree/main/evaluations}} \citep{dhariwal2021diffusion,nichol2021improved,song2023consistency,song2023improved}.

\textbf{Implementation details. } 
We implement SiD based on the EDM \citep{karras2022elucidating} code base and we initialize both the generator $G_{\theta}$ and its score estimation network $f_{\psi}$ by copying the architecture and parameters of the pretrained score network $f_{\phi}$ from EDM \citep{karras2022elucidating}. 
We provide the other implementation details in \Cref{sec:detail}.

\begin{figure}[!th]
\centering
\begin{minipage}{\linewidth} 
\captionof{table}{\small %
Analogous to Table \ref{tab:cifar10_uncond} for ImageNet 64x64 with label conditioning. The Precision and Recall metrics are also included. \label{tab:imagenet}}
\begin{adjustbox}{width=1\linewidth,
center}
\begin{tabular}{clccccc}
 \toprule[1.5pt]
Family & Model &NFE & FID\,($\downarrow$) &Prec.\,($\uparrow$) & Rec.\,($\uparrow$)\\ %
\midrule
 \multirow{2}{*}{Teacher} & \multirow{2}{*}{VP-EDM~\citep{karras2022elucidating}} & 511 & {1.36}\\ %
 &  & 79 & 2.64&0.71&0.67 \\ %
 \midrule
 \multirow{8}{0.2\linewidth}{Direct generation} & RIN~\citep{jabri2022scalable} & 1000& 1.23&& \\
 & DDPM~\citep{ddpm} & 250& 11.00&0.67&0.58 \\ %
 &ADM~\citep{dhariwal2021diffusion}& 250 & 2.07 &{0.74}&{0.63}\\
 & DPM-Solver-3~\citep{lu2022dpmsolver} & 50 &  17.52&&\\
  &HSIVI-SM~\citep{yu2023hierarchical} & 15 &15.49&& \\
 & U-ViT~\citep{bao2022all} & 50 & 4.26&&\\
 & DiT-L/2~\citep{peebles2023scalable} & 250 & 2.91&&\\
 & LEGO~\citep{zheng2023learning} & 250 & 2.16&&\\
 & iCT~\citep{song2023improved} & 1 & 4.02&0.70&0.63 \\ 
 & iCT-deep~\citep{song2023improved} & 1 & 3.25&0.72&0.63 \\ 
\midrule
\multirow{17}{*}{Distillation}
 & PD~\citep{salimans2022progressive} &2 & 8.95&0.63&\textbf{0.65} \\
 & PD~\citep{salimans2022progressive} &1 & 15.39&0.59&0.62 \\
 &G-distill~\citep{meng2023distillation} ($w$=1.0) & 1 & 7.54&& \\
  &G-distill~\citep{meng2023distillation} ($w$=0.3) & 8 & 2.05 &&\\
 &BOOT \citep{gu2023boot} & 1& 16.3&0.68&0.36\\
 &PID~\citep{tee2024physics}  & 1 & 9.49 &&\\
 & DFNO~\citep{zheng2023fast} & 1 & 7.83 &&0.61 \\
 & CD-LPIPS \citep{song2023consistency}&2 & 4.70&0.69&0.64 \\
 & CD-LPIPS \citep{song2023consistency}&1 & 6.20 &0.68&0.63\\
 &Diff-Instruct~\citep{luo2023diffinstruct}&1 & 5.57 &&\\
 &TRACT~\citep{berthelot2023tract}&2 & 4.97 &&\\
 &TRACT~\citep{berthelot2023tract}&1 & 7.43&& \\
 &DMD~\citep{yin2023onestep} &1 & 2.62&& \\
&CTM~\citep{kim2023consistency}&1&{1.92}&& 0.57\\
&CTM~\citep{kim2023consistency}&2&\textbf{\textit{1.73}}&& 0.57\\
& DMD (\textit{w.o.} KL)~\citep{yin2023onestep} &1 & 9.21 &&\\
& DMD (\textit{w.o.} \textit{reg}.)~\citep{yin2023onestep} &1 & 5.61&& \\
&SiD (ours), $\alpha=1.0$&1 & 2.022$\pm$\scriptsize{0.031}&\textbf{\textit{0.73}}&\textbf{\textit{0.63}} \\ 
&SiD (ours), $\alpha=1.2$ &1 %
&\textbf{1.524}$\pm$\scriptsize{0.009} &\textbf{0.74}&\textbf{\textit{0.63}}\\ 
\bottomrule[1.5pt]
\end{tabular}
\end{adjustbox}%
\captionsetup[subfloat]{} 
 \centering
 {\includegraphics[width=.8\linewidth]{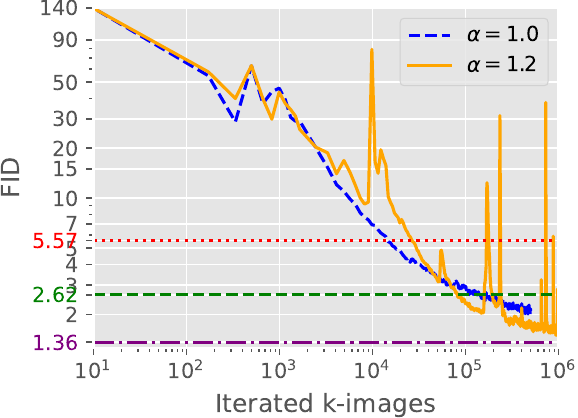}
 }
 \vspace{-2mm}
\captionof{figure}{\small Analogous plot to Fig.~\ref{fig:convergence_speed} for ImageNet 64x64. The batch size is 8192. See the results of batch size 1024 in Fig.~\ref{fig:imagenet_1024}.
 }
 \vspace{-2mm}
 \label{fig:imagenet_convergence_speed}
\end{minipage}
\end{figure}

\begin{figure}[!ht]
\centering
\begin{minipage}{\linewidth} 
\captionof{table}{\small
Analogous to Table \ref{tab:cifar10_uncond} for 
FFHQ 64x64.\label{tab:ffhq}}
\begin{adjustbox}{width=.85\textwidth,
center}
\begin{tabular}{clcc}
\toprule[1.5pt]
Family & Model &NFE & FID~($\downarrow$) \\
\midrule
 Teacher & {VP-EDM~\citep{karras2022elucidating}} & 79 & 2.39 \\ %
\midrule
\multirow{2}{*}{Diffusion} & VP-EDM~\citep{karras2022elucidating} & 50 & 2.60 \\ %
 & Patch-Diffusion~\citep{wang2023patch} & 50 & 3.11\\ %
\midrule
\multirow{3}{*}{Distillation}
&BOOT \citep{gu2023boot} & 1& 9.0\\ 
&SiD (ours), $\alpha=1.0$& 1& 
\textbf{\textit{1.710}} $\pm$ 0.018\\ 
&SiD (ours), $\alpha=1.2$ &1 & \textbf{1.550} $\pm$ 0.017\\
\bottomrule[1.5pt]
\end{tabular}
\end{adjustbox}%
\captionsetup[subfloat]{} 
 \centering
 {\includegraphics[width=.75\linewidth]{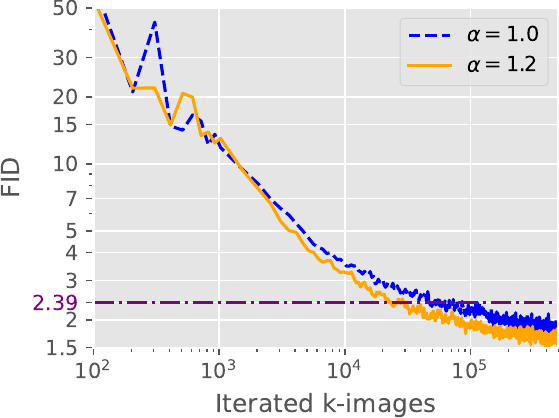}\label{fig:ffhq}}
 \vspace{-3mm}
 \captionof{figure}{\small Analogous plot to Fig.~\ref{fig:convergence_speed} for FFHQ 64x64. The batch size is 512.
 }
 \label{fig:ffhq_convergence_speed}
\end{minipage}
\vspace{2mm}

\centering
\begin{minipage}{\linewidth}
\captionof{table}{\small 
Analogous to Table \ref{tab:cifar10_uncond} for 
AFHQ-v2 64x64.\label{tab:afhq}}
\begin{adjustbox}{width=.85\textwidth,
center}
\begin{tabular}{clcc}
\toprule[1.5pt]
Family & Model &NFE & FID~($\downarrow$) \\
\midrule
 Teacher & {VP-EDM~\citep{karras2022elucidating}} & 79 & 1.96 \\ %
\midrule
\multirow{2}{*}{Distillation}
&SiD (ours), $\alpha=1.0$& 1& \textbf{{1.628}}$\pm$\scriptsize{0.017} \\ 
&SiD (ours), $\alpha=1.2$ &1 & \textbf{\textit{1.711}}$\pm$\scriptsize{0.020}\\
\bottomrule[1.5pt]
\end{tabular}
\end{adjustbox}%
\captionsetup[subfloat]{} 
 \centering
 {\includegraphics[width=.75\linewidth]{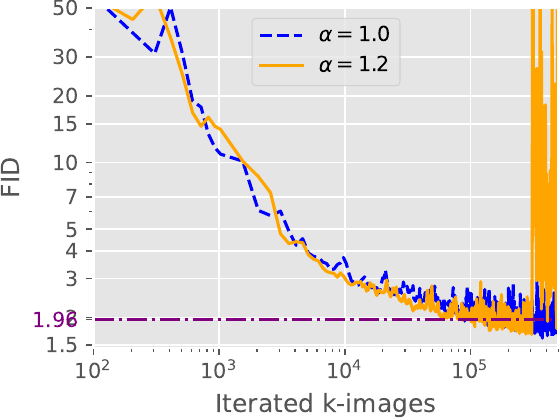}\label{fig:afhq}}
 \vspace{-2mm}
\captionof{figure}{\small Ananagous plot to Fig.~\ref{fig:convergence_speed} for AFHQ-v2 64x64. The batch size is 512.
 }
 \label{fig:afhq_convergence_speed}
\end{minipage}
\vspace{-2mm}
\end{figure}

\textbf{Ablation Study and Parameter Settings. } 
We provide ablation studies and discuss parameter settings in Appendix~\ref{sec:ablation}.

\subsection{Benchmark Performance}
Our comprehensive evaluation compares SiD against leading deep generative models, encompassing both distilled diffusion models and those built from scratch. Random images generated by SiD in a single step are displayed in Figs.\,\ref{fig:cifar10_sample_uncond}-\ref{fig:ffhq_images} in the Appendix.

The comparative analysis, detailed in Tables\,\ref{tab:cifar10_uncond}-\ref{tab:afhq} and illustrated in Figs.\,\ref{fig:convergence_speed}-\ref{fig:afhq_convergence_speed}, underlines the single-step SiD generator's proficiency in leveraging the insights from the pretrained EDM (teacher diffusion model) across a variety of benchmarks, including CIFAR-10 (both conditional and unconditional formats), ImageNet 64x64, FFHQ 64x64, and AFHQ-v2 64x64. Remarkably, the SiD-trained generator surpasses the EDM teacher in nearly all tested environments, showcasing its enhanced performance not just relative to the original multi-step teacher model but also against a broad spectrum of cutting-edge models, from traditional multi-step diffusion models to the latest single-step distilled models and GANs. The sole deviation in this pattern occurs with ImageNet 64x64, where SiD, at $\alpha=1.2$, attains an FID of 1.524, which is  exceeded by \citet{jabri2022scalable}'s RIN at 1.23 FID with 1000 steps and the teacher model VP-EDM's 1.36 FID with 511 steps. %

Our assessment of SiD across various benchmarks has established, with the exception of ImageNet 64x64, potentially the first instance, to our knowledge, where a data-free diffusion distillation method outperforms the teacher diffusion model using just a single generation step. This remarkable outcome implies that reverse sampling, which utilizes the pretrained score function for generating images across multiple steps and naturally accumulates discretization errors during reverse diffusion, might not be as efficient as a single-step distillation process. The latter, by sidestepping error accumulation, could theoretically align perfectly with the true data distribution when the model-based score-matching loss is completely minimized.

Among the single-step generators we've evaluated, CTM \citep{kim2023consistency} is SiD's closest competitor in terms of generation performance. Despite the tight competition, SiD not only surpasses CTM but is also noteworthy for its independence from training data. In contrast, CTM's performance relies on access to training data and is augmented by the inclusion of an auxiliary GAN loss. This distinction significantly amplifies SiD's value, particularly in contexts where accessing the original training data is either restricted or impractical, and where data-specific GAN adjustments are undesirable.

In summary, SiD not only stands out in terms of performance metrics but also simplifies the distillation process remarkably, operating without the need for real data. It sets itself apart by employing a notably straightforward distillation approach, unlike the complex multi-stage distillation strategy seen in \citet{salimans2022progressive}, the dependency on pairwise regression data in \citet{yin2023onestep}, the use of additional GAN loss in \citet{kim2023consistency}, or the need to access training data outlined in \citet{song2023consistency}.

\textbf{Training Iterations. }
In exploring SiD's performance threshold, we initially process 500 million SiD-generated synthetic images across most benchmarks. For CIFAR-10 with label conditioning, this figure increases to 800 million synthetic images for SiD with $\alpha=1.2$. In the case of ImageNet 64x64, we extend the training for SiD with  $\alpha=1.2$ to involve 1 billion synthetic images. Through this extensive training, SiD demonstrates superior performance over the EDM teacher model across all evaluated benchmarks, with the sole exception of ImageNet 64x64, where  EDM utilized 511 NFE. While we note a gradual slowing down in the rate of FID improvements, the limit of potential further reductions is not clear, indicating that with more iterations, SiD might eventually outperform EDM on ImageNet 64x64 as well.

It's noteworthy that to eclipse the achievements of rivals like Diff-instruct and DMD, SiD requires significantly fewer synthetic images than the 500 million mark, thanks to its rapid FID reduction rate. This decline often continues without evident stagnation,  surpassing the teacher model's performance before the conclusion of the training. We delve into this aspect further below.

\textbf{Convergence Speed. } 
Our SiD generator, designed for distilling pretrained diffusion models, rapidly achieves the capability to generate photo-realistic images in a single step. This efficiency is showcased in Fig.\,\ref{fig:imagenet_progress} for the EDM model pretrained on ImageNet 64x64 and in Fig.\,\ref{fig:cifar_alpha_image} for CIFAR 32x32 (unconditional).
The performance of the SiD method is further highlighted in Figs.\,\ref{fig:convergence_speed}-\ref{fig:afhq_convergence_speed}, where the x-axis represents the thousands of images processed during training. These figures track the FID's evolution across four datasets for both \(\alpha=1\) and \(\alpha=1.2\), demonstrating a roughly linear relationship between the logarithm of the FID and the logarithm of  the number of processed images. This relationship indicates that FID decreases exponentially as distillation progresses, a trend that is observed or expected to eventually slow down and approach a steady state.

For instance, on CIFAR-10 (unconditional), SiD outperforms both Diff-Instruct and DMD after processing under 20M images, achievable within fewer than 10 hours on 16 A100-40GB GPUs, or 20 hours on 8  V100-16GB GPUs. In the case of ImageNet 64x64 with a batch size of 1024 and $\alpha=1.0$, SiD exceeds Progressive Distillation (FID 15.39) of \citet{salimans2022progressive} after only around 500k generator-synthesized images (equivalent to roughly 500 iterations with a batch size of 1024), achieving FIDs lower than 5 after 7.5M images, below 4 after 13M, and under 3 after 31M images. It outperforms Diff-Instruct with fewer than 7M images processed and DMD with under 40M images. When using a larger batch size of 8192, SiD's convergence is slower, yet it attains lower FIDs: with $\alpha=1$, it outstrips Diff-Instruct after processing less than 20M images (under 20 hours on 16 A100-40GB  GPUs), and with $\alpha=1.2$, it beats DMD after fewer than 90M images (in under 45 hours on 16 A100-40GB GPUs).

\textbf{Limitations. }
Despite setting a new benchmark in diffusion-based generation, SiD entails the simultaneous management of three networks during the distillation process: the pretrained score network $f_{\phi}$, the generator score network $f_{\psi}$, and the generator $f_{\theta}$, which, in this study, are maintained at equal sizes. This setup demands more memory compared to traditional diffusion model training, which only necessitates retaining $f_{\phi}$. However, the memory footprint of the two additional networks could be notably reduced by employing LoRA \citep{hu2022lora} for both $f_{\psi}$ and $f_{\theta}$, a possibility we aim to explore in future research.

Relative to Diff-Instruct, acknowledged for its memory and computational efficiency in distillation, as detailed in Table\,\ref{tab:Hyperparameters} in the Appendix for 16 A100-40GB GPUs, the memory allocation per GPU of SiD has seen a rise of around 50\% for ImageNet 64x64 and about 70\% for CIFAR-10, FFHQ, and AFHQ. The iteration time has increased by approximately 28\% for CIFAR-10 and ImageNet 64x64, and by roughly 36\% for the FFHQ and AFHQ datasets.
This increase is because Diff-Instruct does not require computing score gradients, as defined in \eqref{eq:scoregrad}. By contrast, SiD necessitates computing score gradients, involving backpropagation through both the pretrained and generator score networks—a step not needed in Diff-Instruct—leading to about a one-third increase in computing time per iteration.

\section{Conclusion}
We present Score identity Distillation (SiD), an innovative method that transforms pretrained diffusion models into a single-step generator. By employing semi-implicit distributions, SiD aims to accomplish distillation through the minimization of a model-based score-matching loss that aligns the scores of diffused real and generative distributions across different noise intensities. Experimental outcomes underscore SiD's capability to significantly reduce the Fréchet inception distance with remarkable efficiency and outperform established generative approaches. This superiority extends across various conditions, including those using single or multiple steps, those requiring or not requiring access to training data, and those needing additional loss functions in image generation.

\section*{Acknowledgments}
The authors would like to thank Dr. Zhuoran Yang and Weijian Luo for their valuable comments and suggestions. M. Zhou, H. Zheng, and Z. Wang acknowledge the support of NSF-IIS 2212418 and NIH-R37 CA271186.

\section*{Impact Statement}
The positive aspect of distilled diffusion models lies in their potential to save energy and reduce costs. By simplifying and compressing large models, the deployment of distilled models often requires less computational resources, making them more energy-efficient and cost-effective. This can lead to advancements in sustainable AI practices, especially in resource-intensive applications.

However, the negative aspect arises when considering the ease of distilling models trained on violent or pornographic data. This poses significant ethical concerns, as deploying such models may inadvertently facilitate the generation and dissemination of harmful content. The distillation process, intended to transfer knowledge efficiently, could unintentionally amplify and perpetuate inappropriate patterns present in the original data. This not only jeopardizes user safety but also raises ethical and societal questions about the responsible use of AI technology. Striking a balance between the positive gains in energy efficiency and the potential negative consequences of distilling inappropriate content is crucial for the responsible development and deployment of AI models. Stringent ethical guidelines and oversight are essential to mitigate these risks and ensure the responsible use of distilled diffusion models.


\bibliographystyle{icml2024}
\bibliography{reference.bib,ref.bib}

\newpage
\appendix
\onecolumn

\begin{center}
   \Large{\textbf {Appendix for Score identity Distillation}} 
\end{center}


\section{Ablation Study and Parameter Settings}\label{sec:ablation}

\textbf{Impact of $\alpha$.} We conduct an ablation study to examine the impact of  \(\alpha\) on SiD. In Fig.~\ref{fig:cifar_alpha_image}, we investigate a range of \(\alpha\) values \([-0.25, 0.0, 0.5, 0.75, 1.0, 1.2, 1.5]\) during SiD training on CIFAR10-unconditional and visualize the changes in generation as training progresses under each \(\alpha\). For instance, the last row illustrates the generation results when 10.24 million images (equivalent to 40,000 iterations with a batch size of 256) are processed by SiD.
In Fig.~\ref{fig:cifar_ablate}, we illustrate the evolution of the FID and IS from iterations 0 to 8000 (corresponding to 0 to 1.024 million images), where the first plot depicts the IS evolution, while the second plot shows the trajectory of FID.

The results indicate a stable performance of the model when \(\alpha\) varies from 0 to 1.2. A negative value of \(\alpha\) results in large FIDs. This observation supports our analysis in \cref{sec:4.2} that directly optimizing \(\mathcal{L}_{\theta}^{(1)}\) given by \eqref{eq:L_theta1} may not lead to meaningful improvement, as our loss, shown in \eqref{eq:loss_combine}, is \(\mathcal{L}_{\theta}^{(2)} - \alpha \mathcal{L}_{\theta}^{(1)}\). As \(\alpha\) increases within the tested range, we observe a gradual improvement in IS and FID performance, peaking at \(\alpha=1\) or \(\alpha=1.2\).
Based on these findings, we select \(\alpha=1\) or \(\alpha=1.2\) for all our experiments, although a more refined grid search on $\alpha$ might reveal even better performance outcomes.

\textbf{Setting of $\beta_1$.} We investigate the $\beta_1$ parameter of the Adam optimizer for the generator score network $f_{\psi}$ and the generator $G_{\theta}$ by setting it as either $\beta_1=0$, the value used in StyleGAN2 \citep{Karras2019stylegan2}  and Diff-Instruct \citep{luo2023diffinstruct}, or $\beta_1=0.9$, a commonly used value. We find that setting $\beta_1=0.9$ for $f_{\psi}$ often does not result in convergence, so we retain $\beta_1=0$ for $f_{\psi}$ for all datasets. For learning $G_{\theta}$, we did not observe significant benefits between setting $\beta_1=0$ and $\beta_1=0.9$, except for the FFHQ dataset, where the FID improved by more than 0.15 when changing from $\beta_1=0$ to $\beta_1=0.9$. Therefore, we set $\beta_1=0.9$ for $G_{\theta}$ in FFHQ while retaining $\beta_1=0$ for all other datasets.

\textbf{Batch Size for ImageNet 64x64.}
For ImageNet 64x64, we initially set the batch size to 1024 and observed an exponential decline in FID until it suddenly diverged upon reaching or surpassing 2.62, the FID obtained by DMD \citep{yin2023onestep}. The exact reason for this divergence is still unclear, but we suspect it may be related to the FP16 precision used during optimization. While switching to FP32 could potentially address the issue, we have not explored this option due to its much higher computational and memory costs.

Instead, we increased the overall batch size from 1024 to 8192 (while keeping the batch per GPU unchanged at 16, requiring more gradient accumulation rounds) and reduced the learning rate from 5e-6 to 4e-6. Under $\alpha=1$, we observed stable performance, while under $\alpha=1.2$, we observed occasional spikes in FID. Upon examining the generations corresponding to these spikes, as shown in the fifth image of Fig.~\ref{fig:imagenet_progress2} and Fig.~\ref{fig:spikes} in the Appendix, we found interesting patterns where certain uncommon features, such as nests containing birds, were exaggerated. However, with a batch size as large as 8192, these occasional spikes did not seem to significantly impact the overall declining trend, which was roughly log-log linear initially and gradually leveled off. With that said, when the batch size was reduced to 1024, the sudden divergence could potentially be caused by such a spike, as observed in Fig.~\ref{fig:imagenet_1024}.

The drawback of using a larger batch size in this case is that it takes SiD longer to outperform Diff-instruct and DMD, as clearly shown by comparing the FID trajectories in Figs.~\ref{fig:imagenet_convergence_speed} and \ref{fig:imagenet_1024}.
Although it's feasible to develop more advanced strategies, including progressively increasing the batch size, annealing the learning rate, and implementing gradient clipping, we'll reserve these for future study.

\newpage

\begin{figure*}[ht]
    \setlength\tabcolsep{1pt} %
    
    \centering

    \rotatebox{90}{\!\!\!\!\!\!\!\!\!\!\!\!\!\!\! 102k (400)} %
    \begin{tabular}{ccccccc}
    $\alpha=-0.25$ &$\alpha=0$& $\alpha=0.5$ & $\alpha=0.75$ & $\alpha=1$ & $\alpha=1.2$& $\alpha=1.5$\\
        \includegraphics[width=0.13\linewidth]{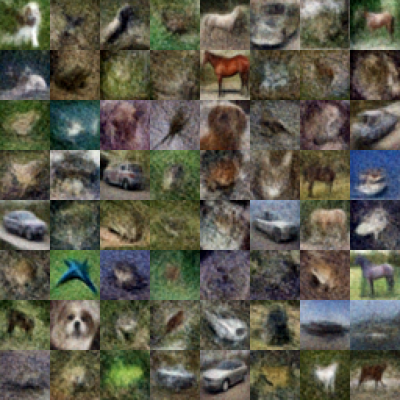} &
        \includegraphics[width=0.13\linewidth]{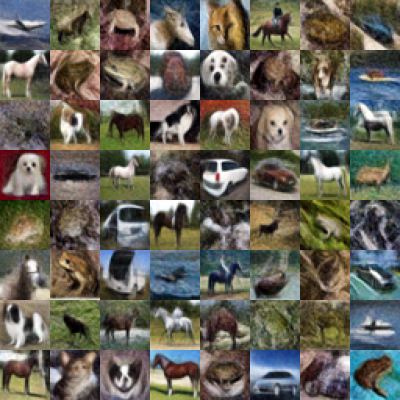} &
        \includegraphics[width=0.13\linewidth]{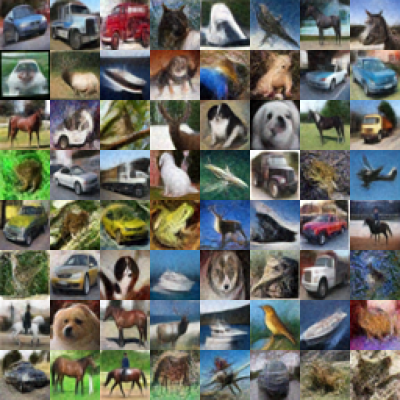} &
        \includegraphics[width=0.13\linewidth]
        {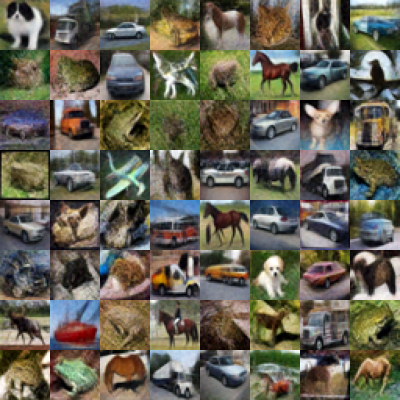} &
        \includegraphics[width=0.13\linewidth]{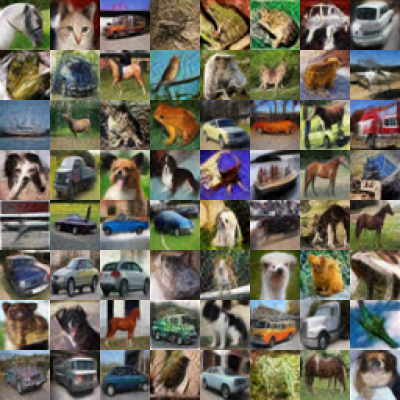} &
        \includegraphics[width=0.13\linewidth]{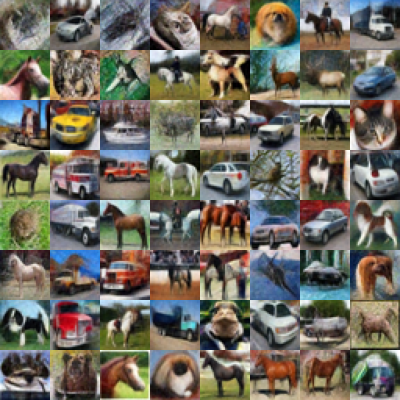} &
        \includegraphics[width=0.13\linewidth]{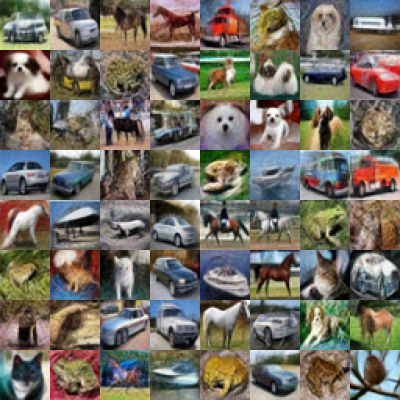} \\
    \end{tabular}

   \vspace{-3pt} %

    \rotatebox{90}{\!\!\!\!\!\!\!\!\!\!\!\!\!\! 512k (2000)} %
    \begin{tabular}{ccccccc}
        \includegraphics[width=0.13\linewidth]{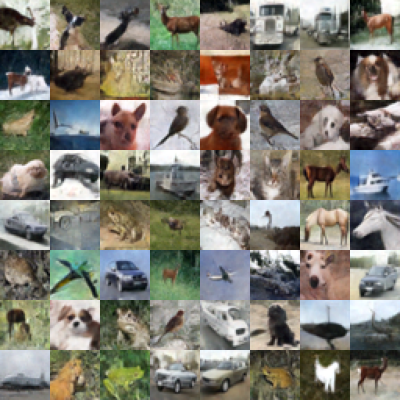} &
        \includegraphics[width=0.13\linewidth]{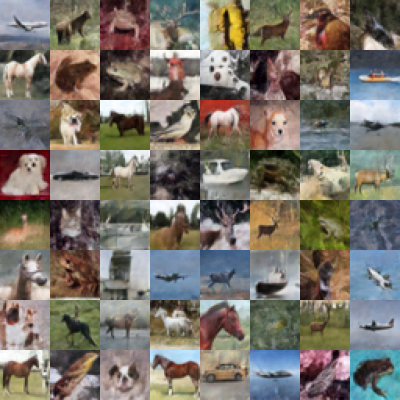} &
        \includegraphics[width=0.13\linewidth]{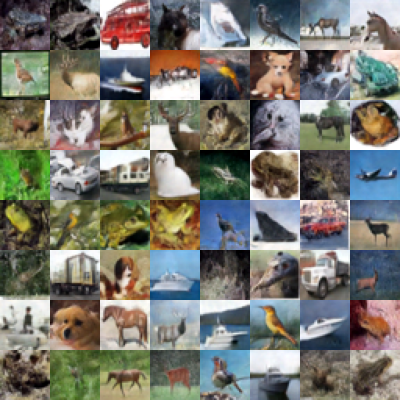} &
        \includegraphics[width=0.13\linewidth]{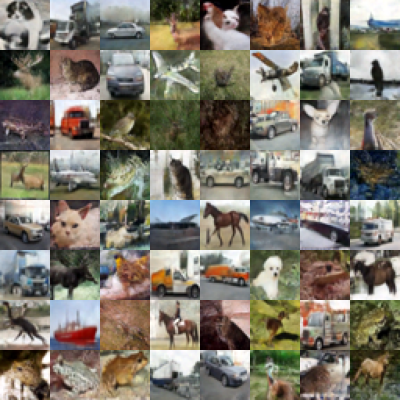} &
        \includegraphics[width=0.13\linewidth]{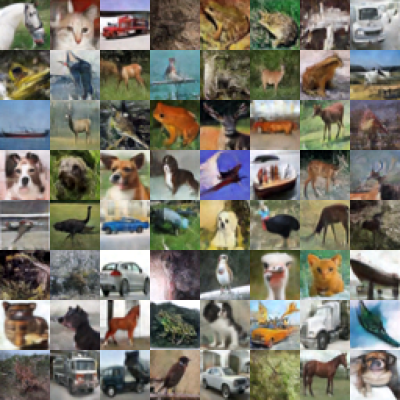} &
        \includegraphics[width=0.13\linewidth]{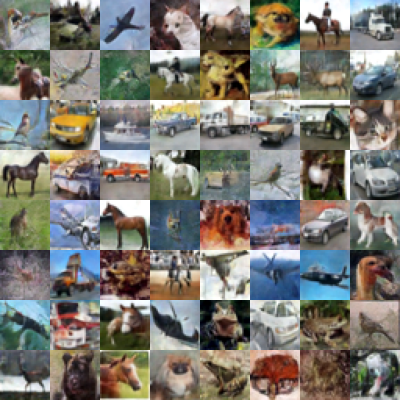} &
        \includegraphics[width=0.13\linewidth]{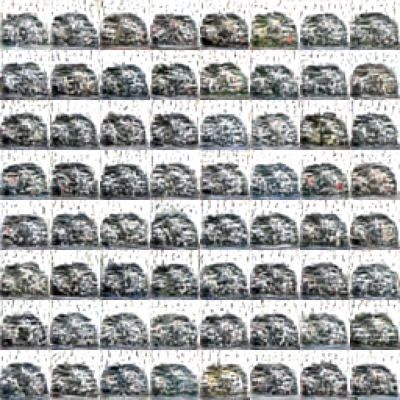} \\
    \end{tabular}

    \vspace{-3pt} %

    \rotatebox{90}{\!\!\!\!\!\!\!\!\!\!\!\! \!\!\!\!\!  1024k (4000)} %
    \begin{tabular}{ccccccc}
        \includegraphics[width=0.13\linewidth]{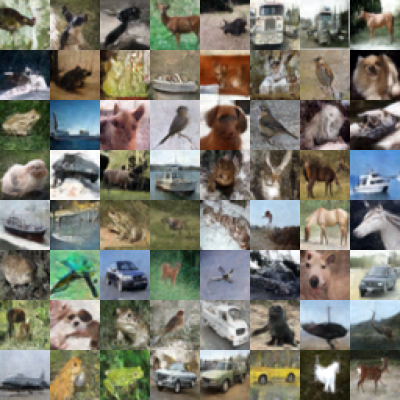} &
        \includegraphics[width=0.13\linewidth]{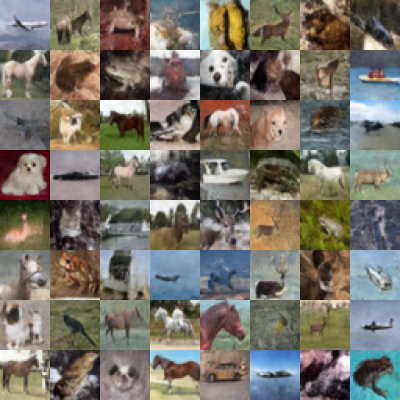} &
        \includegraphics[width=0.13\linewidth]{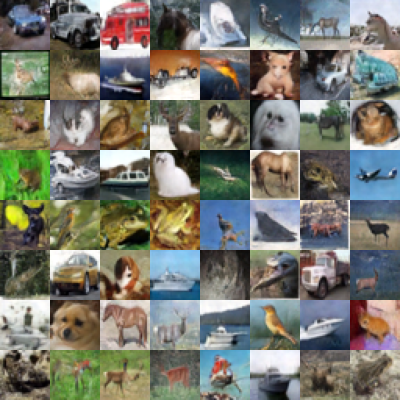} &
        \includegraphics[width=0.13\linewidth]{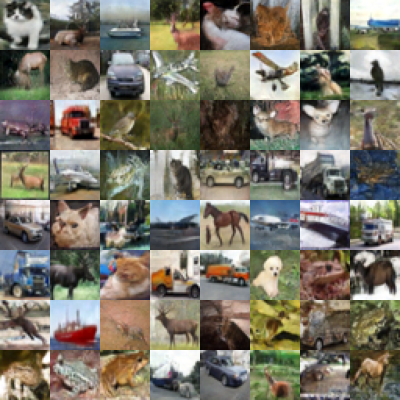} &
        \includegraphics[width=0.13\linewidth]{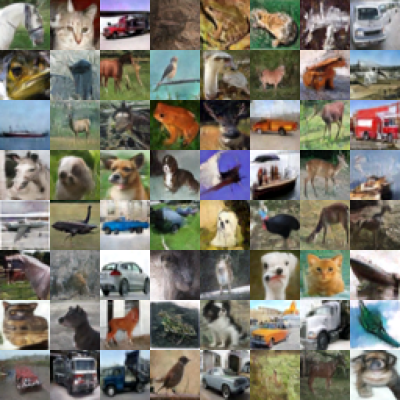} &
        \includegraphics[width=0.13\linewidth]{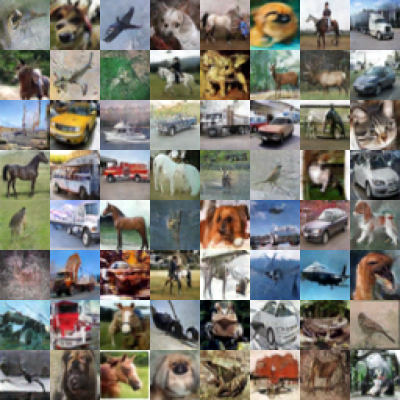} &
        \includegraphics[width=0.13\linewidth]{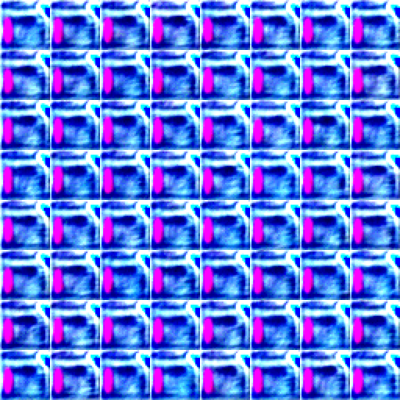} \\
    \end{tabular}

    \vspace{-3pt} %

    \rotatebox{90}{\!\!\!\!\!\!\!\!\!\!\!\!  \!\!\!\!\!\!\!\!\! 10240k (40000)} %
    \begin{tabular}{ccccccc}
        \includegraphics[width=0.13\linewidth]{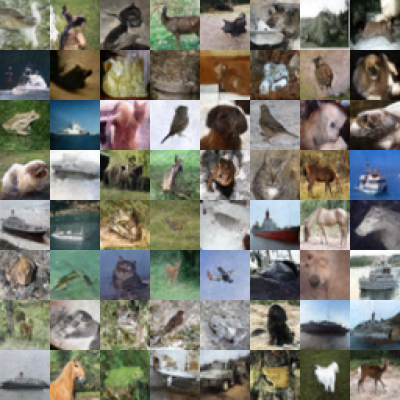} &
        \includegraphics[width=0.13\linewidth]{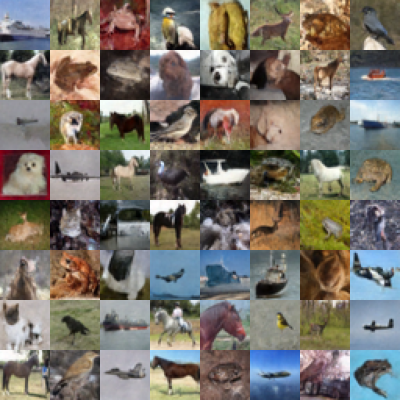} &
        \includegraphics[width=0.13\linewidth]{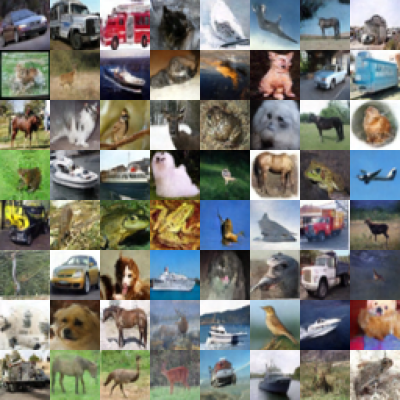} &
        \includegraphics[width=0.13\linewidth]{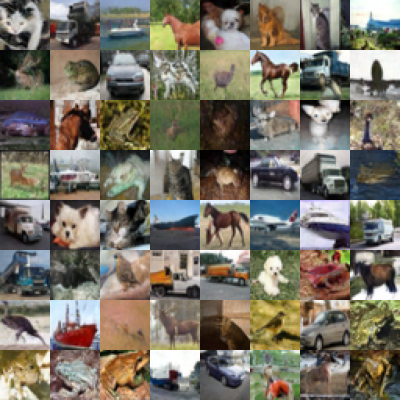} &
        \includegraphics[width=0.13\linewidth]{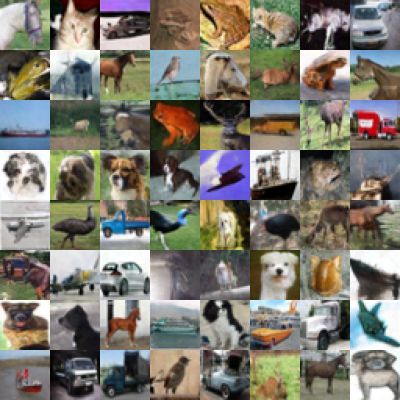} &
        \includegraphics[width=0.13\linewidth]{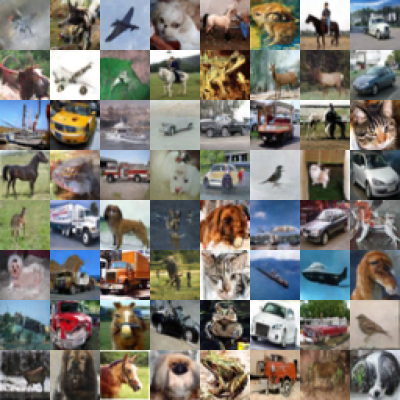} &
        \includegraphics[width=0.13\linewidth]{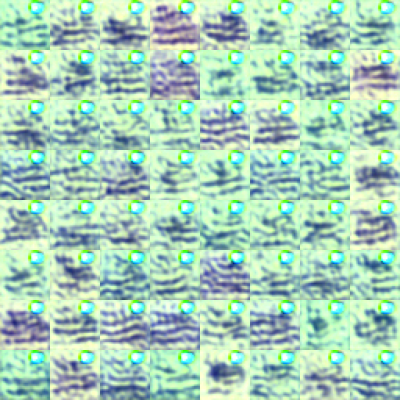}
    \end{tabular}
\caption{
Ablation Study of $\alpha$: The SiD generator, configured with various $\alpha$ values, was trained with its own synthesized images at a batch size of 256. The results, sorted by specific $\alpha$ values, are displayed in columns. Sequentially from top to bottom, the rows are labeled with both the total number of training images and the corresponding number of iterations, denoted as ``number of images (iterations).'' This labeling approach indicates the cumulative count of fake images utilized during training, corresponding to iterations of 400, 2,000, 4,000, and 40,000, progressing from the first row to the last. 
Across the 
$\alpha$  values of 0.5, 0.75, 1, and 1.2, minor differences are noted in both the Inception Score (IS) and visual quality, yet the Fréchet Inception Distance (FID) shows notable variations, as detailed in Fig.\,\ref{fig:cifar_ablate}.
 }
 \label{fig:cifar_alpha_image}
\end{figure*}

\begin{figure*}[!ht]
\centering
\includegraphics[width=.33\linewidth]{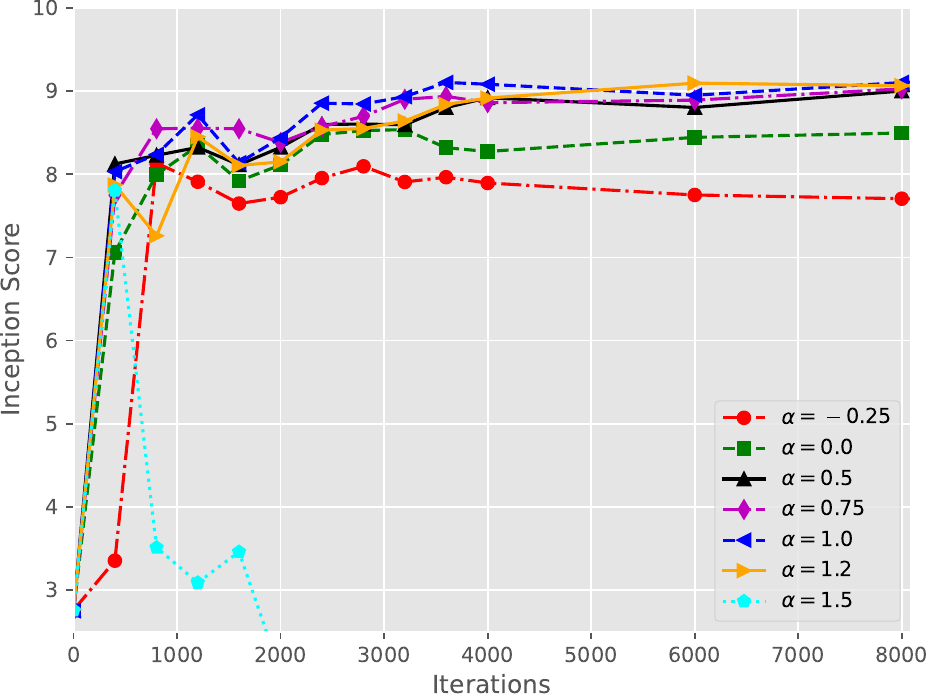}
~~~~~~~~~~~~~~~~~~~~~
\includegraphics[width=.33\linewidth]{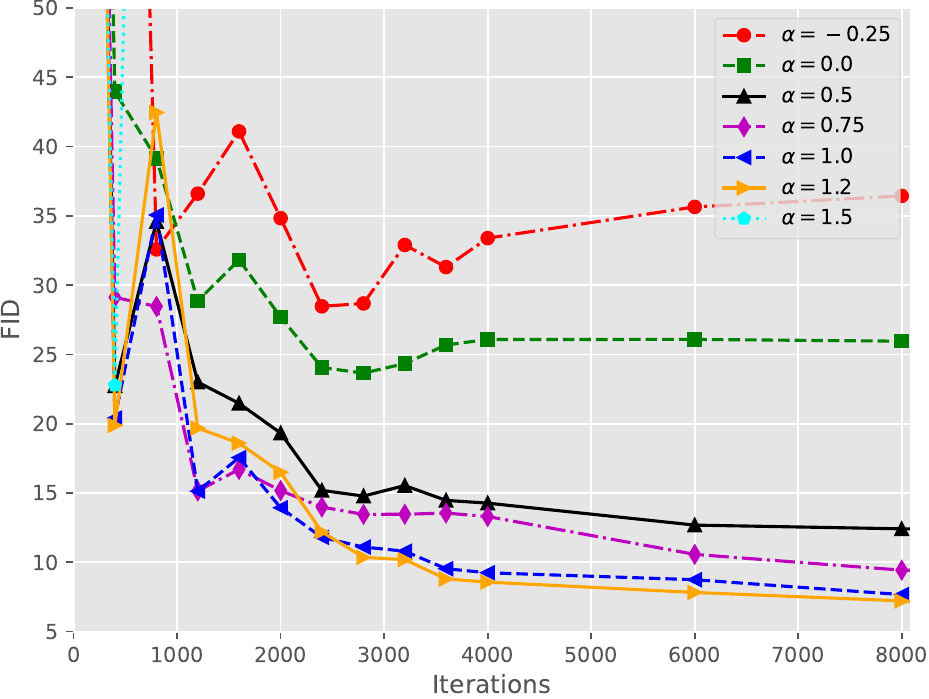}
 \caption{\small Ablation Study of $\alpha$: Each plot illustrates the relation between the performance, measured by Inception Score and FID \textit{vs}. the number of training iterations 
 during the distillation of the EDM model pretrained on CIFAR-10 (unconditional), across varying values of $\alpha$.
 The study underscores the impact of $\alpha$ on both training efficiency and generative fidelity, leading us to select $\alpha \in \{1.0, 1.2\}$ for all subsequent experiments.
 }
 \label{fig:cifar_ablate}
\end{figure*}

\newpage

\section{Algorithm Box}
\begin{algorithm}[ht]
\caption{Score identity Distillation (SiD)}\label{alg:sid}
\begin{algorithmic}
\small
\STATE \textbf{Input:} Pretrained score network $f_\phi$, generator $G_\theta$, generator score network $f_\psi$, $\sigma_{\text{init}}=2.5$, $t_{\max}=800$, $\alpha=1.2$
\STATE \textbf{Initialization} $\theta \leftarrow \phi, \psi \leftarrow \phi$ 
\REPEAT
\STATE Sample $\zv \sim \cN(0,  \mathbf{I})$ and let $\xv_g= G_\theta(\sigma_{\text{init}}\zv)$; Sample $t\sim p(t)$ and $\epsilonv_t\sim \mathcal{N}(0,\mathbf{I})$, and let $\xv_t = \xv_g + \sigma_t \epsilonv_t$; Update $\psi$ with \Cref{eq:obj-psi}:
\STATE ~~~~$\hat{\cL}_{\psi} = {\gamma(t)} %
\|f_{\psi}(\xv_t,t)-\xv_g\|_2^2 $
\STATE ~~~~$\psi = \psi - \eta \nabla_\psi \hat{\cL}_{\psi}$
\STATE where the timestep distribution $t\sim p(t)$, noise level $\sigma_t$,
and 
weighting function $\gamma(t)$ %
are defined as in
\citet{karras2022elucidating}. 

\STATE Sample $\zv \sim \cN(0,  \mathbf{I})$ and let $\xv_g= G_\theta(\sigma_{\text{init}}\zv)$; 
Sample $t\sim \mbox{Unif}[0,t_{\max}/1000]$, compute $\sigma_t$ with \Cref{eq:sigma_t},  compute $\omega_t$ with \Cref{eq:weight}, and let $\xv_t = \xv_g + \sigma_t \epsilonv_t$; Update $G_\theta$ with \Cref{eq:obj-theta}:
\STATE ~~~~$\tilde{\cL}_{\theta} =(1 -\alpha) \frac{\omega(t)}{{\sigma_t^4} } \|f_{\phi}(\xv_t,t)-f_{\psi}(\xv_t,t)\|_2^2 $
\STATE ~~~~~~~~~~$+ \frac{\omega(t)}{{\sigma_t^4} } (f_{\phi}(\xv_t,t)-f_{\psi}(\xv_t,t))^T(f_{\psi}(\xv_t,t)-\xv_g)$
\STATE ~~~~$\theta = \theta - \eta \nabla_\theta \tilde{\cL}_{\theta}$
\UNTIL{the FID plateaus or the training budget is exhausted}
\STATE \textbf{Output:} $G_\theta$
\normalsize
\end{algorithmic}
\end{algorithm} %


\section{Proofs}\label{sec:proof}

\begin{proof}[Proof of Tweedie's formula]
For Gaussian diffusion, we have \eqref{eq:Gaussian_conditional_score}, which we explore to derive the identity shown below. 
While $p_{\theta}(\xv_t)$ often does not have an analytic form, exploiting its semi-implicit construction, its score can be expressed as
\ba{
\nabla_{\xv_t} \ln p_{\theta}(\xv_t) &= \frac{\int \nabla_{\xv_t}q(\xv_t\given \xv_g)p_{{\theta}}(\xv_g) \mathrm   d\xv_g }{p_{\theta}(\xv_t)}\notag\\
&= \frac{\int q(\xv_t\given \xv_g) \nabla_{\xv_t} \ln q(\xv_t\given \xv_g) p_{\theta}(\xv_g)\mathrm  d\xv_g }{p_{\theta}(\xv_t)}\notag\\
&= -\frac{\int q(\xv_t\given \xv_g)\frac{\xv_t-a_t\xv_g}{\sigma_t^2} p_{\theta}(\xv_g)\mathrm  d\xv_g }{p_{\theta}(\xv_t)}\notag\\
&= -\frac{\xv_t}{\sigma_t^2} +\frac{a_t}{\sigma_t^2} \frac{\int \xv_g q(\xv_t\given \xv_g) p_{\theta}(\xv_g)\mathrm  d\xv_g }{p_{\theta}(\xv_t)}\notag\\
&= -\frac{\xv_t}{\sigma_t^2} +\frac{a_t}{\sigma^2} \int \xv_g q(\xv_g\given \xv_t) \mathrm d\xv_g\notag\\
&= -\frac{\xv_t}{\sigma_t^2} +\frac{a_t}{\sigma_t^2} \E[\xv_g\given \xv_t]
}
%
%
%
Therefore, we have
\ba{
\E[\xv_g\given \xv_t] = \frac{\xv_t+\sigma_t^2 \nabla_{\xv_t} \ln q_g(\xv_t)}{a_t}
}
which is known as the Tweedie's formula. Setting $a_t=1$ recovers the identity presented in the main body of the paper. %
\end{proof}

\begin{proof}[Proof of Identity~\ref{projected_score}.]
\ba{
&\E_{\xv_t\sim p_{\theta}(\xv_t)}[u^T(\xv_t) \nabla_{\xv_t}\ln p_{\theta}(\xv_t)]\notag\\ &= \E_{\xv_t\sim p_{\theta}(\xv_t)}\left[u^T(\xv_t) \frac{\nabla_{\xv_t} p_{\theta}(\xv_t)}{p_{\theta}(\xv_t)}\right] \notag\\
&= \int{u^T(\xv_t)\nabla_{\xv_t} p_{\theta}(\xv_t)}\mathrm d \xv_t\notag\\
& = \int u^T(\xv_t) \int \nabla_{\xv_t} q(\xv_t\given \xv_g) p_{\theta}(\xv_g)  \mathrm d \xv_g \mathrm d\xv_t\notag\\
& = \int u^T(\xv_t) \int q(\xv_t\given \xv_g) \nabla_{\xv_t} \ln q(\xv_t\given \xv_g) p_{\theta}(\xv_g)\mathrm  d \xv_g \mathrm d\xv_t\notag\\
& = \E_{ (\xv_t,\xv_g)\sim q(\xv_t\given \xv_g) p_{\theta}(\xv_g)}[u^T(\xv_t) \nabla_{\xv_t} \ln q(\xv_t\given \xv_g)].
}
\end{proof}

\begin{proof}[Proof of Theorem~\ref{thm:project_SM}]
 Expanding the $L_2$ norm, we have

\ba{
&\E_{\xv_t\sim p_{\theta}(\xv_t)}\|S(\xv_t)-\nabla_{\xv_t}\ln p_{\theta}(\xv_t)\|_2^2 \notag\\
&=\frac{1}{\sigma_t^2}\E_{\xv_t\sim p_{\theta}(\xv_t)}[(\E[\xv_0\given \xv_t]-\E[\xv_g\given \xv_t])^T(S(\xv_t)-\nabla_{\xv_t}\ln p_{\theta}(\xv_t))]\notag\\
&=\underbrace{\frac{1}{\sigma_t^2}\E_{\xv_t\sim p_{\theta}(\xv_t)}\left[(\E[\xv_0\given \xv_t]-\E[\xv_g\given \xv_t])^TS(\xv_t)\right]}_{\textcircled{1}}
-\underbrace{\frac{1}{\sigma_t^2}\E_{\xv_t\sim p_{\theta}(\xv_t)}[(\E[\xv_0\given \xv_t]-\E[\xv_g\given \xv_t])^T\nabla_{\xv_t}\ln p_{\theta}(\xv_t)]}_{\textcircled{2}}
}
denote 
\ba{
\textcircled{1} & =\frac{1}{\sigma_t^2} \E_{\xv_t\sim p_{\theta}(\xv_t)}\left[\delta_{\phi,\psi^*(\theta)}(\xv_t)^T(\E[\xv_0\given\xv_t]-\xv_t)\right] \notag\\
&= \frac{1}{\sigma_t^2}\E_{\xv_t\sim p_{\theta}(\xv_t)}[\delta_{\phi,\psi^*(\theta)}(\xv_t)^T\E[\xv_0\given \xv_t]]-\frac{1}{\sigma_t^2}\E_{\xv_t\sim p_{\theta}(\xv_t)}[\delta_{\phi,\psi^*(\theta)}(\xv_t)^T\xv_t)] 
}
\ba{
\textcircled{2} &= \E_{\xv_g\sim p_{\theta}(\xv_g)}\E_{ \xv_t\sim q(\xv_t\given \xv_g,t)}[\delta_{\phi,\psi^*(\theta)}(\xv_t)^T\nabla_{\xv_t}\ln 
q(\xv_t\given \xv_g))] \notag\\
&= \frac{1}{\sigma_t^2}\E_{\xv_g\sim p_{\theta}(\xv_g)}\E_{ \xv_t\sim q(\xv_t\given \xv_g,t)}[\delta_{\phi,\psi^*(\theta)}(\xv_t)^T(\xv_g-\xv_t)]\notag\\
&= \frac{1}{\sigma_t^2}\E_{\xv_g\sim p_{\theta}(\xv_g)}\E_{ \xv_t\sim q(\xv_t\given \xv_g,t)}[\delta_{\phi,\psi^*(\theta)}(\xv_t)^T\xv_g]- \frac{1}{\sigma_t^2}\E_{ \xv_t\sim p_{\theta}(\xv_t)}[\delta_{\phi,\psi^*(\theta)}(\xv_t)^T\xv_t]
}
Therefore we have
\ba{
L = \textcircled{1}-\textcircled{2} %
= \frac{1}{\sigma_t^2}\E_{\xv_g\sim p_{\theta}(\xv_g)}\E_{ \xv_t\sim q(\xv_t\given \xv_g,t)}[\delta_{\phi,\psi^*(\theta)}(\xv_t)^T(\E[\xv_0\given \xv_t]-\xv_g)]
}
\end{proof}

\section{Analytic study of the toy example}
\label{sec:toy}

We prove the conclusions in \Cref{prop:toy,prop:toy2}. 
Given $p_\text{data}(\xv_0) = \mathcal{N}(\mathbf{0},\mathbf{I})$, $p_\theta(\xv_g) = \mathcal{N}(\theta,\mathbf{I})$, $q(\xv_t\given \xv_0) = \mathcal N(\xv_t;\xv_0,\sigma_t^2\mathbf{I})$, and  $q(\xv_t\given \xv_g) = \mathcal N(\xv_t;\xv_g,\sigma_t^2\mathbf{I})$, we have $p_{\text{data}}(\xv_t)=\mathcal{N}(0,(1+\sigma_t^2)\mathbf{I})$ and $p_{\theta}(\xv_t)=\mathcal{N}(\theta,(1+\sigma_t^2)\mathbf{I})$. 
The optimal value of $\theta$ would be $\theta^*=0$. The score can be expressed as
\bas{
&S(\xv_t) = \nabla_{\xv_t}\ln p_{\text{data}}(\xv_t) = -\frac{\xv_t}{1+\sigma_t^2} \\
&\nabla_{\xv_t}\ln p_{\theta}(\xv_t) = -\frac{\xv_t-\theta}{1+\sigma_t^2}.
}
Hence, the difference between the scores is $\delta_{\phi,\psi^*(\theta)}(\xv_t) = - \frac{\theta}{1+\sigma_t^2}$. 
By applying Tweedie's formula as described in Identities~\ref{identity1} and \ref{identity3}, we obtain
\bas{
&f_{\phi}(\xv_t,t)=\E[\xv_0\given \xv_t] 
= \frac{\xv_t}{1+\sigma_t^2} \\
&\E[\xv_g\given \xv_t] = \xv_t\frac{1}{1+\sigma_t^2} + \theta \frac{\sigma_t^2}{1+\sigma_t^2}
}

By assumption we have
$$
f_{\psi}(\xv_t,t) = \xv_t\frac{1}{1+\sigma_t^2} + \psi \frac{\sigma_t^2}{1+\sigma_t^2},
$$
which means $\psi^*(\theta)=\theta$, then by \Cref{eq:delta0} we have
\ba{
{\delta}_{\phi,\psi}(\xv_t) = \sigma_t^{-2}(f_{\phi}(\xv_t,t)-
f_{\psi}(\xv_t,t)) = - \frac{\psi}{1+\sigma_t^2}.
}
Accordingly,
\ba{
\hat{L}_{\theta}^{(1)} = {\delta}_{\phi,\psi}(\xv_t)^T{\delta}_{\phi,\psi}(\xv_t) = \frac{\psi^2}{(1+\sigma_t^2)^2} 
}
Therefore, while $\hat{L}_{\theta} = \delta_{\phi,\psi^*(\theta)}(\xv_t)^T\delta_{\phi,\psi^*(\theta)}(\xv_t)= \frac{\theta^2}{(1+\sigma_t^2)^2}$ would provide useful gradient to learn $\theta$, its naive approximation $\hat{L}_{\theta}^{(1)}$ could fail to provide meaningful gradient. 

We can further compute
\bas{
\hat{L}_{\theta}^{(2)} &= \hat{L}_{\theta}^{(1)} +\frac{{\delta}_{\phi,\psi}(\xv_t)^T(f_{\psi}(\xv_t,t)-\xv_g)}{\sigma_t^2} \notag \\
&= \frac{\psi^2}{(1+\sigma_t^2)^2} - \frac{\psi}{\sigma_t^2(1+\sigma_t^2)}(\xv_t\frac{1}{1+\sigma_t^2} + \psi \frac{\sigma_t^2}{1+\sigma_t^2}-\xv_g)\notag\\
&=\frac{\psi }{(1+\sigma_t^2)^2}\left[\xv_g-\frac{\epsilon_t}{\sigma_t}\right]. 
}
Thus
\bas{
\nabla_\theta \hat{L}_{\theta}^{(2)} &= \frac{\psi }{(1+\sigma_t^2)^2}\nabla_{\theta} G_{\theta}(z) \\
&=- \frac{1}{1+\sigma_t^2} {\delta}_{\phi,\psi}(\xv_t) \nabla_{\theta} G_{\theta}(z) \\
&\approx \frac{1}{1+\sigma_t^2} [\nabla_{\xv_t} \ln p_\theta(\xv_t) - S_\phi(\xv_t)] \nabla_\theta G_\theta(z).
}

\section{Training and Evaluation Details and Additional Results.} \label{sec:detail}
The hyperparameters tailored for our study are outlined in Table \ref{tab:Hyperparameters}, with all remaining settings consistent with those in the EDM code \cite{karras2022elucidating}. The initial development of the SiD algorithm utilized a cluster with 8 Nvidia RTX A5000 GPUs. To support a mini-batch size up to 8192 for ImageNet 64x64, we adopted the gradient accumulation strategy. Extensive evaluations across four diverse datasets were conducted using cloud computation nodes equipped with either 16 Nvidia A100-40GB GPUs, 8 Nvidia V100-16GB GPUs, or 8 Nvidia H100-80GB GPUs, with most experiments performed on Nvidia A100-40GB GPUs.

Comparisons of memory usage and per-iteration computation costs between SiD and Diff-Instruct, utilizing 16 Nvidia A100-40GB GPUs, are detailed in Table \ref{tab:Hyperparameters}.

We note the time and memory costs reported in Table \ref{tab:Hyperparameters} do not include these used to evaluate the Fr\'echet Inception Distance (FID) of the single-step generator during the distillation process.
The FID for the SiD generator, utilizing exponential moving average (ema), was evaluated after processing each batch of 500k generator-synthesized fake images. We preserve the SiD generator that achieves the lowest FID, and to ensure accuracy, we re-evaluate it across 10 independent runs to calculate the corresponding metrics. It's worth noting that some prior studies have reported the best metric obtained across multiple independent random runs, a practice that raises concerns about reliability and reproducibility. We consciously avoid this approach in our work to ensure a more robust and credible evaluation.

\begin{table}[ht]
\caption{Hyperparameter settings and comparison of distillation time and memory usage between Diff-Instruct and SiD on 16 NVIDIA A100 GPUs with 40 GB of memory each.}
\label{tab:Hyperparameters}
\begin{center}
\resizebox{\textwidth}{!}{
\begin{tabular}{cccccc}
\toprule
Method &Hyperparameters & CIFAR-10 32x32 & ImageNet 64x64 & FFHQ 64x64 & AFHQ-v2 64x64 \\
\midrule
&Batch size & 256 & 8192 & 512 & 512 \\
&Batch size per GPU &16 & 16&32&32\\
&\# of GPUs (40G A100)& 16 & 16 & 16 & 16\\
&Gradient accumulation round& 1 & 32 & 1 & 1\\
&Learning rate of ($\psi$, $\theta$) & 1e-5 & 4e-6 & 1e-5 & 5e-6\\
&Loss scaling of ($\psi,\theta$) &\multicolumn{4}{c}{(1,100)}\\
&ema & 0.5 & 2 & 0.5 & 0.5 \\
&fp16 & False & True & True & True \\
&Optimizer Adam (eps) & 1e-8 & 1e-6 & 1e-6 & 1e-6 \\
&Optimizer Adam ($\beta_1$) of $\theta$ &0&0&0.9&0\\
&Optimizer Adam ($\beta_1$) of $\psi$ &\multicolumn{4}{c}{0} \\
&Optimizer Adam ($\beta_2$) &\multicolumn{4}{c}{0.999}\\
&$\alpha$ & \multicolumn{4}{c}{$1.0$ and $1.2$}\\
&$\sigma_{\text{init}}$ & \multicolumn{4}{c}{2.5}\\
&$t_{\max}$ &\multicolumn{4}{c}{800}\\
&augment, dropout, cres  &\multicolumn{4}{c}{The same as in EDM for each corresponding dataset }\\
\midrule
\multirow{3}{*}{Diff-Instruct}&max memory in GB allocated per GPU & 4.4 & 20.4 & 8.1 & 8.1\\
&max memory in GB reserved per GPU & 4.7  & 23.0  &10.8 & 10.8 \\
&$\sim$seconds per 1k images & 1.4 & 2.8& 1.1 & 1.1 \\
\midrule
\multirow{5}{*}{SiD}&max memory in GB allocated per GPU & 7.8  & 31.3 & 17.0 & 17.0 \\
&max memory in GB reserved per GPU & 8.1  & 31.9  &17.2& 17.2 \\
&$\sim$seconds per 1k images & 1.6 & 3.6& 1.3 & 1.3 \\
&$\sim$hours per 10M ($10^4$k) images & 4.4 & 10.0 & 3.6 & 3.6 \\
&$\sim$days per 100M ($10^5$k) images & 1.9 & 4.2& 1.5 & 1.5 \\
&$\sim$days per 500M ($5\times 10^5$k) images & 9.3 & 20.8& 7.5 & 7.5 \\
\bottomrule
\end{tabular}}
\end{center}
\end{table}

\begin{figure}[t]
\centering
{\includegraphics[width=.5\linewidth]{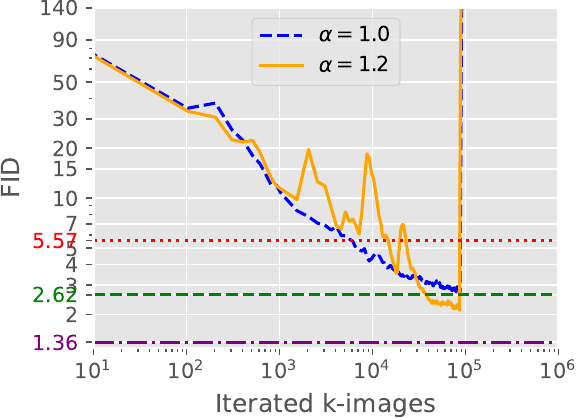}
}
\captionof{figure}{\small Analogous plot to Fig.~\ref{fig:imagenet_convergence_speed} for ImageNet 64x64, where the batch size for SiD is 1024 and learning rate is 5e-6. The FID declines fast until it suddenly diverges. Increasing the batch size to 8192 and lowering the learning rate to 4e-6, as shown in Fig.~\ref{fig:imagenet_convergence_speed},  has alleviated the issue of sudden divergence.
 }
 \label{fig:imagenet_1024}
\end{figure}

\begin{figure*}[h]
\begin{center}
\includegraphics[width=0.137\linewidth]{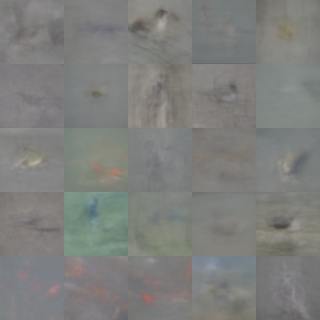}
\includegraphics[width=0.137\linewidth]{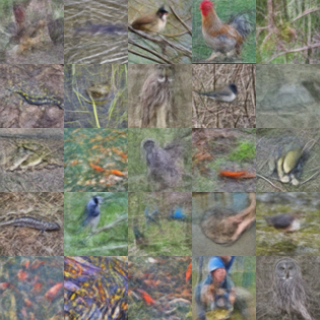}
\includegraphics[width=0.137\linewidth]{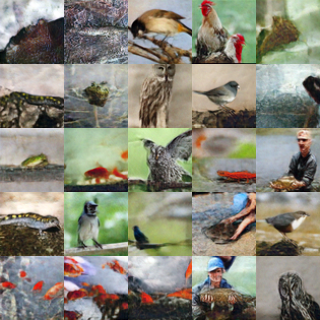}
\includegraphics[width=0.137\linewidth]{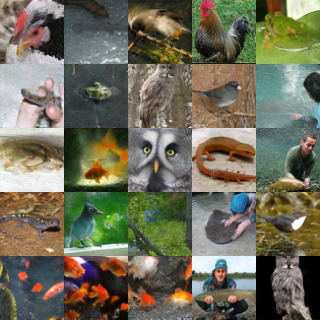}
\includegraphics[width=0.137\linewidth]{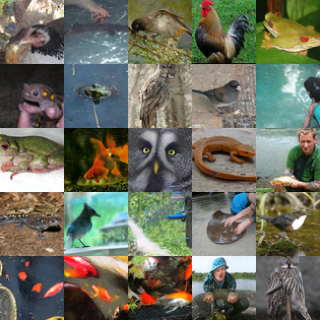}
\includegraphics[width=0.137\linewidth]{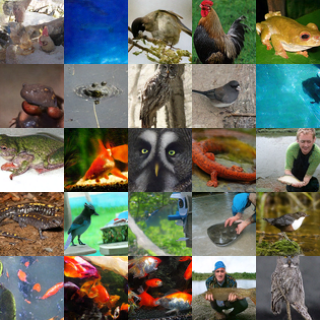}
\includegraphics[width=0.137\linewidth]{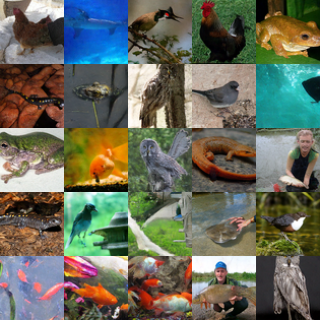}
\end{center}
 \caption{\small
 Similar to Fig.~\ref{fig:imagenet_progress}, this plot showcases the SiD method's efficacy with $\alpha=1.0$, a batch size of 8192, and a learning rate of 4e-6. The images are created using a consistent set of random noises after training the SiD generator with differing numbers of synthesized images, specifically 0, 0.2, 1, 5, 10, 20, and 50 million images. These correspond to approximately 0, 20, 120, 600, 1.2K, 2.5K, and 6.1K training iterations, respectively, displayed sequentially from left to right. The corresponding FIDs at these stages are 153.73, 54.63, 46.07, 11.02, 6.93, 4.68, and 3.34. The progression of FIDs is illustrated by the dashed blue curve in Fig.~\ref{fig:imagenet_convergence_speed}.
 }
 \label{fig:imagenet_progress1}
\end{figure*}

\begin{figure*}[h]
\begin{center}
\includegraphics[width=0.137\linewidth]{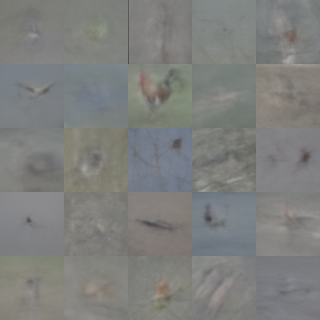}
\includegraphics[width=0.137\linewidth]{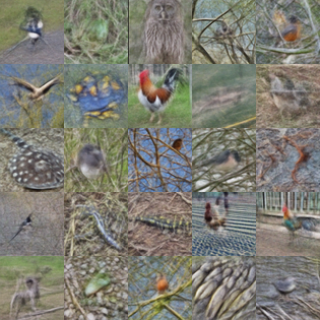}
\includegraphics[width=0.137\linewidth]{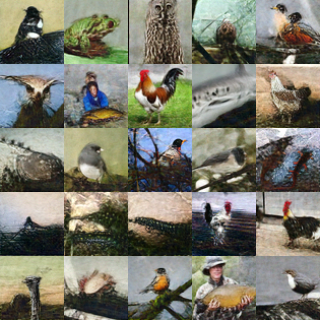}
\includegraphics[width=0.137\linewidth]{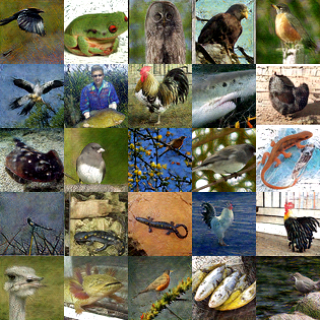}
\includegraphics[width=0.137\linewidth]{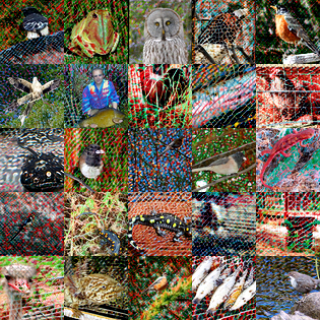}
\includegraphics[width=0.137\linewidth]{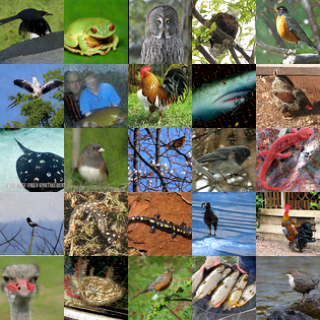}
\includegraphics[width=0.137\linewidth]{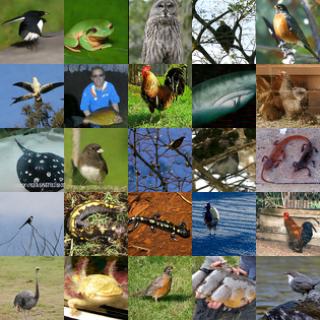}
\end{center}
 \caption{\small
 Analogous plot to Fig.~\ref{fig:imagenet_progress1} for SiD with an adjusted parameter $\alpha=1.2$. The corresponding FIDs are 154.05, 57.63, 43.55, 16.89, 78.92, 7.45, and 3.22.  The progression of FIDs is illustrated by the solid orange curve in Fig.~\ref{fig:imagenet_convergence_speed}.
 }
 \label{fig:imagenet_progress2}
\end{figure*}

\begin{figure*}[h]
\begin{center}
\includegraphics[width=0.137\linewidth]{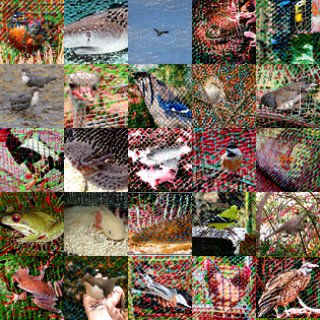}
\includegraphics[width=0.137\linewidth]{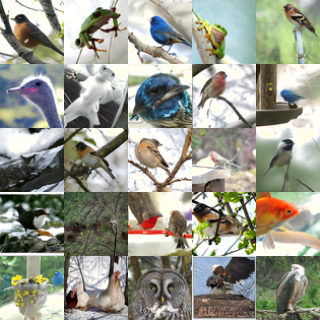}
\includegraphics[width=0.137\linewidth]{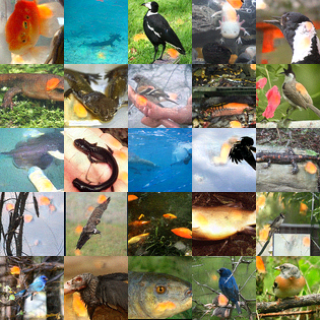}
\includegraphics[width=0.137\linewidth]{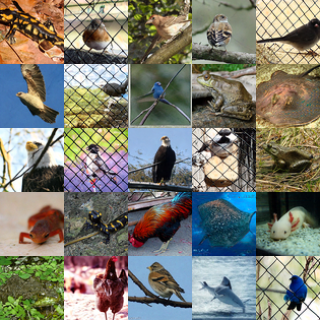}
\includegraphics[width=0.137\linewidth]{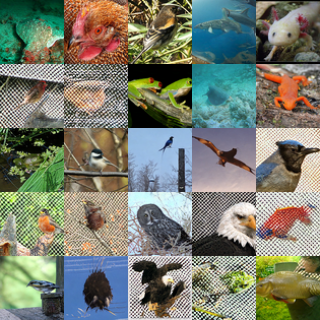}
\includegraphics[width=0.137\linewidth]{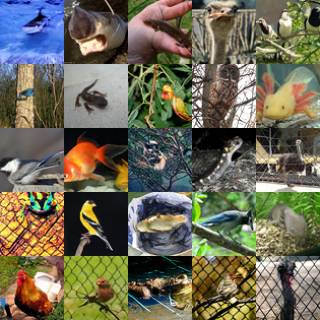}
\includegraphics[width=0.137\linewidth]{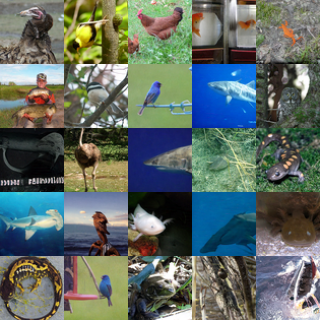}
\end{center}
 \caption{\small
All but the last subplot consist of example SiD generated images corresponding to the spikes of the solid orange curve in Fig.~\ref{fig:imagenet_convergence_speed}, which depicts the evolution of FIDs of SiD with $\alpha=1.2$. These spikes are observed after processing around 10, 55, 17, 23, 73, and 88 million images. The last subplot displays SiD generated images using the generator with the lowest FID.
 }
 \label{fig:spikes}
\end{figure*}

\newpage
\begin{figure*}[h]
 \centering
{\includegraphics[width=1\linewidth]{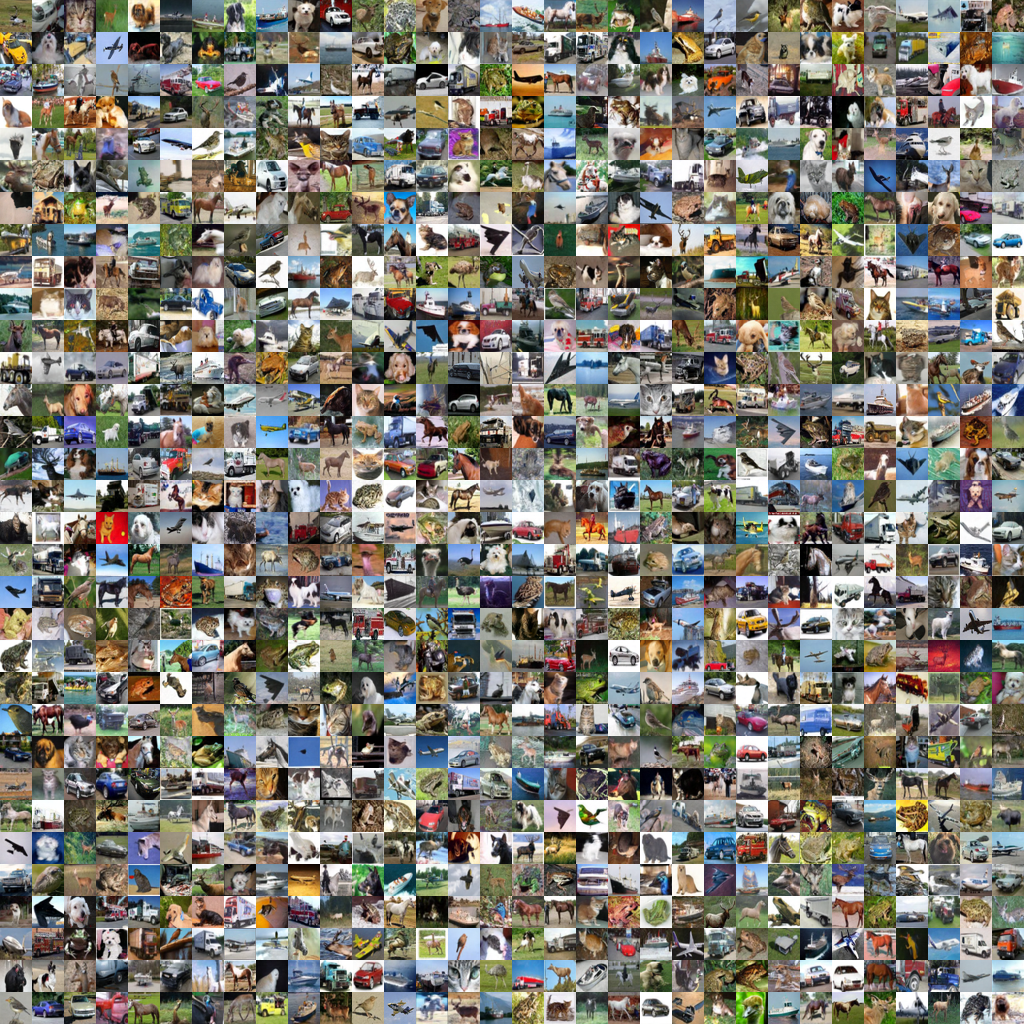}
}
 \caption{
 Uncoditional CIFAR-10 32X32 random images generated with SiD (FID: 1.923).
 }
 \label{fig:cifar10_sample_uncond}
\end{figure*}

\begin{figure*}[!h]
 \centering
{\includegraphics[width=1\linewidth]{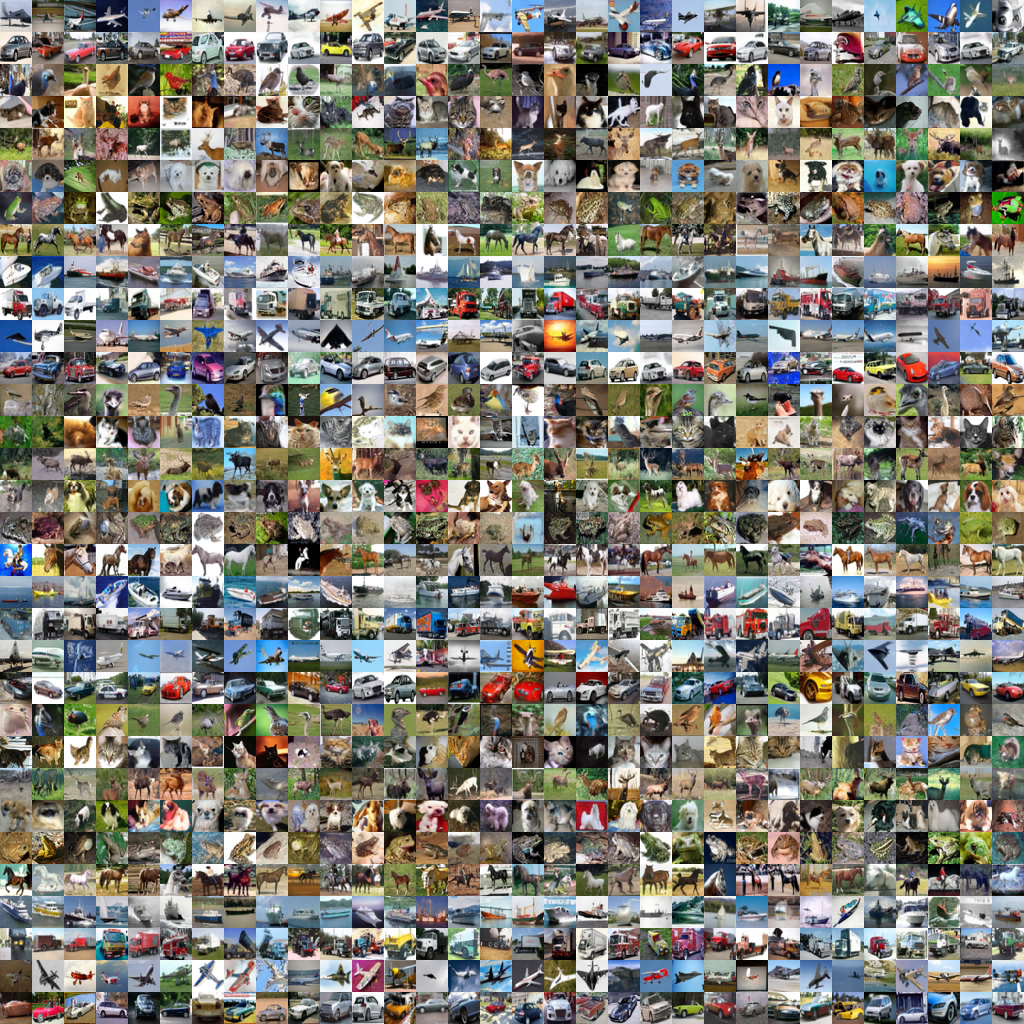}}
 \caption{
 Label conditioning CIFAR-10 32X32 random images generated with SiD (FID: 1.710)
 }
 \label{fig:cifar10_sample_cond}
\end{figure*}
\begin{figure*}[!h]
 \centering
{\includegraphics[width=1\linewidth]{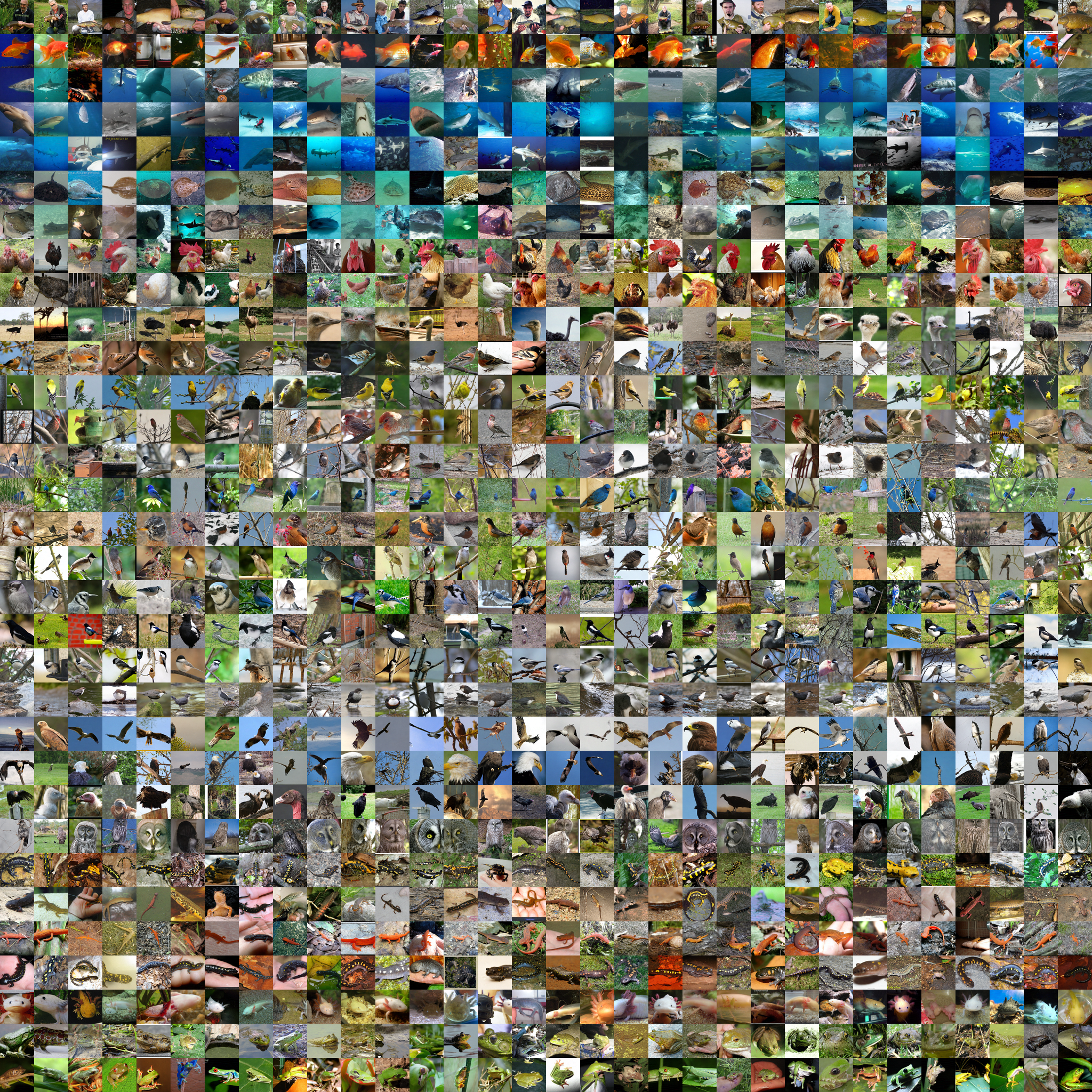}}
 \caption{
 Label conditioning ImageNet 64x64 random images generated with SiD (FID: 1.524)
 }
 \label{fig:imagenet_images}
\end{figure*}

\begin{figure*}[!h]
 \centering
{\includegraphics[width=1\linewidth]{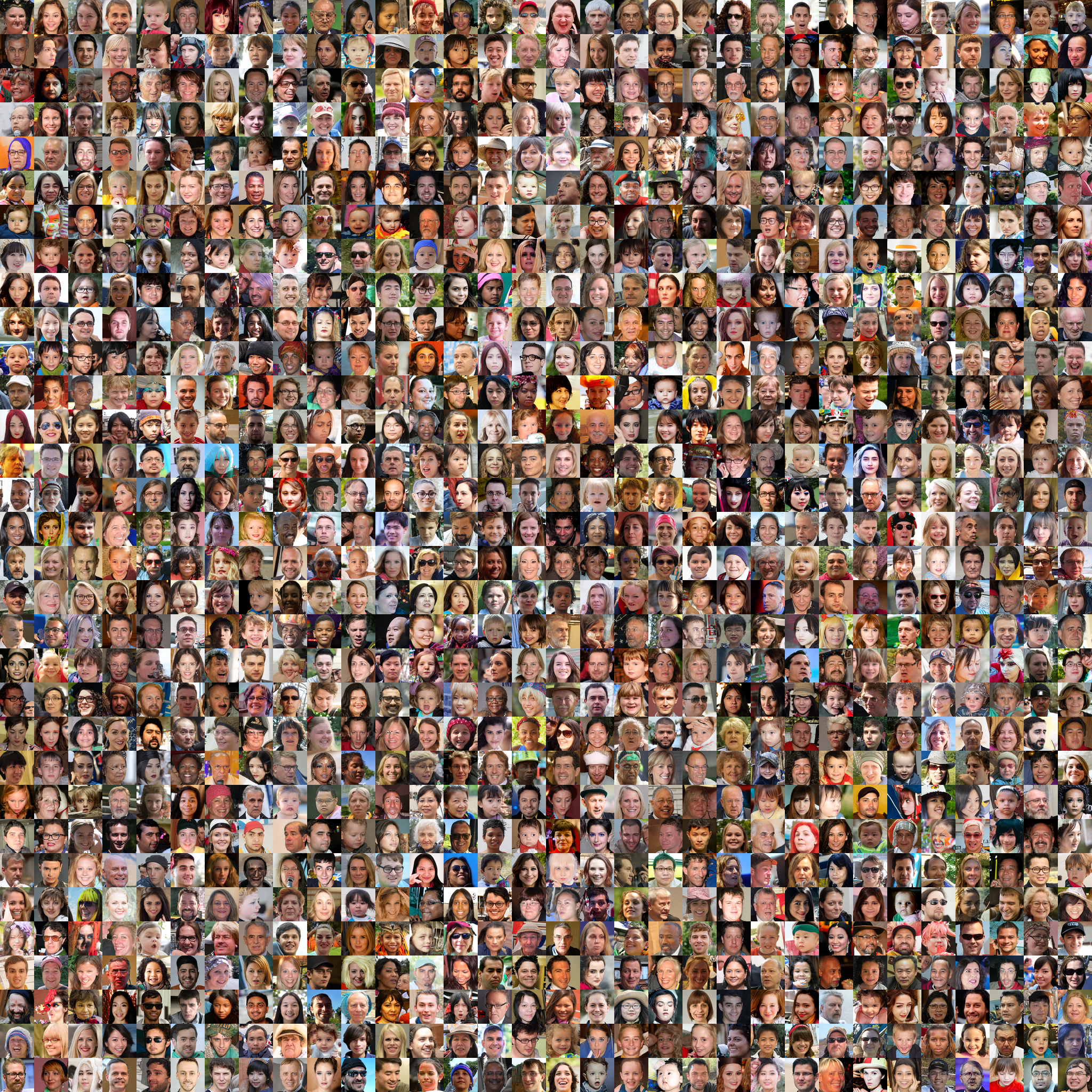}
}
 \caption{
 FFHQ 64X64 random images generated with SiD 
 (FID: 1.550)
 }
 \label{fig:ffhq_images}
\end{figure*}

\begin{figure*}[!h]
 \centering
 {\includegraphics[width=1\linewidth]{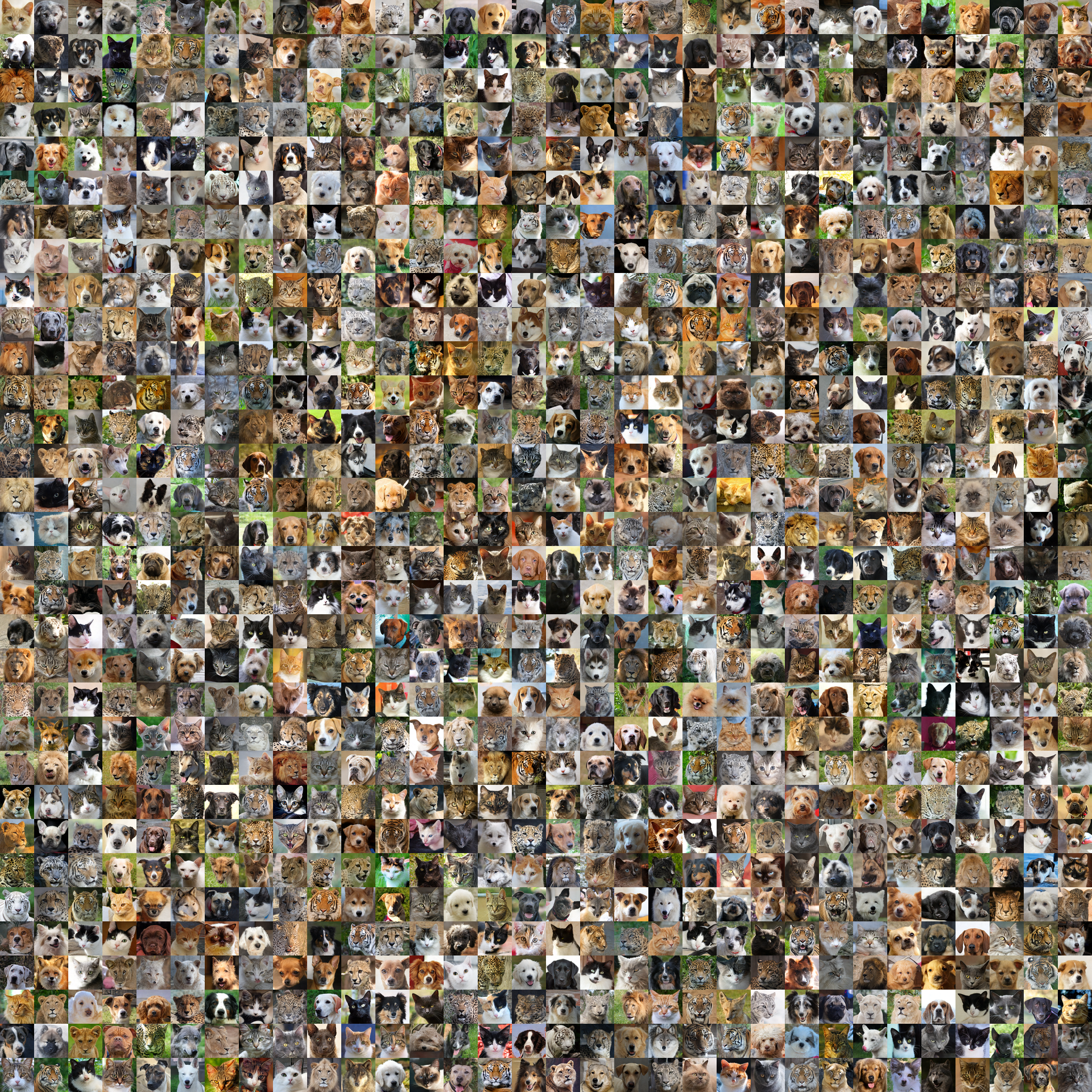}}
 \caption{
 AFHQ-V2 64X64 random images generated with SiD (FID: 1.628)
 }
 \label{fig:afhq_images}
\end{figure*} 
\end{document}